%% file: main.tex
\pdfoutput=1
\documentclass{article}

\input{preamble/preamble.tex}
\input{preamble/preamble_math}
\input{preamble/definitions_basic}

\input{preamble/minimalist_biblatex}

\crefformat{equation}{(#2#1#3)}
\crefformat{figure}{Figure~#2#1#3}
\crefname{example}{Example}{Examples}
\crefname{lemma}{Lemma}{Lemmas}
\crefname{cor}{Corollary}{Corollaries}
\crefname{theorem}{Theorem}{Theorems}
\crefname{assumption}{Assumption}{Assumptions}

\crefname{assumption}{Assumption}{Assumptions}

\usepackage{enumitem} 
\usepackage[separate-uncertainty=true,multi-part-units=single]{siunitx} 

\declaretheoremstyle[
spacebelow=\parsep, spaceabove=\parsep, mdframed={backgroundcolor=gray!10!white, hidealllines=true, innertopmargin=8pt, innerbottommargin=4pt, skipabove=8pt, skipbelow=10pt, nobreak=true } ]{grayboxed}

\crefname{gassumption}{Assumption}{Assumptions}

\usepackage{thm-restate}

\usepackage{xcolor}
\usepackage{bm}

\usepackage[affil-it]{authblk}

\newcommand{\pos}[1]{\bm{#1}}
\newcommand{\neu}[1]{\textcolor{gray}{#1}}
\newcommand{\nege}[1]{\textbf{\textcolor{red}{#1}}}

\title{Anchor Divergence for Semantic Geometry in Contrastive Learning}

\date{}
\author[1]{Akash Kannan}
\author[1]{Kiho Park}
\author[1]{Victor Veitch}
\affil[1]{University of Chicago}

\begin{document}
\maketitle

\begin{abstract}
	This paper concerns how semantic context determines geometry in learned vector representations. Similarity is typically measured using cosine similarity, which provides a single fixed geometry. Semantic similarity, however, is inherently context dependent: two images may be similar because they depict the same object, share a visual style, or are relevant to the same clinical finding.
	We show that contrastive representations naturally encompass a family of geometries that can be specialized to particular semantic structure.
	The key idea is to use an interplay between contrastive learning, exponential families, and information geometry to establish a correspondence between probability distributions over ``anchors'' and Bregman geometries on the representation space.
	We use this correspondence to define \emph{Anchor Divergences}, a method for specifying context-specific semantic geometries on fixed representations. Under this correspondence, modeling the anchor distribution models the geometry itself. 
	Experiments on retrieval show that anchor divergences provide an effective and efficient way to specify context-specific semantic similarity. Code is available at \href{https://github.com/sky1712/Anchor-Divergence}{github.com/sky1712/Anchor-Divergence}.
\end{abstract}

\section{Introduction}

Modern AI systems encode semantic structure in dense vector representations. 
This paper asks how similarity between such representations should be measured. 
We focus on contrastive models such as CLIP and its successors \citep{radford2021learning}, which are trained to score matched query--key pairs above mismatched ones using cosine-based softmax objectives. 
It is therefore natural to use cosine similarity as the geometry of the learned space. But this creates a basic mismatch: cosine provides a single fixed notion of similarity, while semantic similarity is inherently context dependent.

For concreteness, we focus on image similarity in vision--language models, and study how semantic similarity should be defined between pairs of images. \Cref{fig:intro_semantic_neighborhoods} illustrates the fundamental limitation of cosine similarity in this setting. A query image may share different semantic attributes with different database images, while cosine implicitly fixes how these sources of variation contribute to similarity. The relevant notion, however, depends on context: we may care about the same dog breed, background, or outfit. 
Understanding how to specify such context-specific similarity is important both foundationally---the question of how dense representations encode semantics is the heart of (mechanistic) interpretability---and practically. For example, when retrieving chest X-rays similar to a query image, we may want similarity to reflect specific pathologies rather than unrelated patient characteristics or acquisition artifacts; see \Cref{fig:chexpert_qualitative}. A fixed cosine geometry has no reason to privilege the semantic aspect relevant to the task.

\begin{figure}[t]
    \centering
    \includegraphics[
        width=0.86\linewidth
    ]{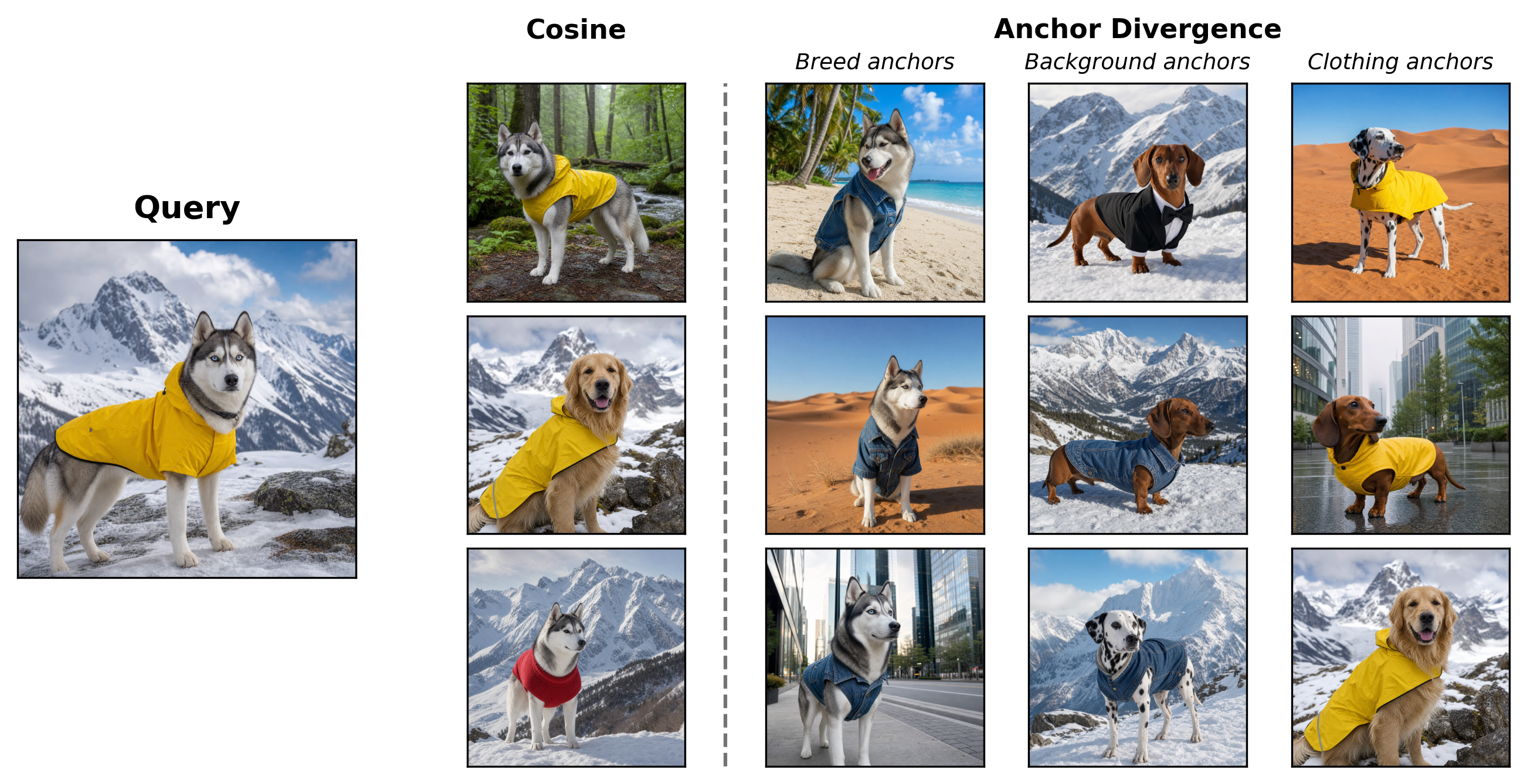}
    \caption{\textbf{Anchor divergences induce context-specific neighborhoods for the same query.}
    Cosine similarity gives a single fixed ranking. In contrast, changing the semantic anchors changes which aspect of the query defines its neighborhood. For example, breed anchors retrieve other Huskies while allowing background and clothing to vary; background and clothing anchors behave analogously. The underlying image representations remain fixed throughout.}
    \label{fig:intro_semantic_neighborhoods}
\end{figure}

The main contribution of this paper is to show that the contrastive structure naturally defines a \emph{family} of semantic geometries on the image representation space, each induced by a probability distribution over semantically relevant ``anchors.'' 
This correspondence leads to \emph{anchor divergences}, a method for specifying context-specific semantic geometries on fixed representations; see \Cref{fig:intro_semantic_neighborhoods}.

Informally, the starting idea is to specify the semantics of interest using a set of text anchors. 
For example, in the chest X-ray setting, the anchors might consist of clinical descriptions of pathologies---the variation in the text includes only the aspects we care about.
Relative to any such anchor set, the contrastive model uses the same softmax structure that appears during training to map each image to a distribution over the anchors, based on their image--text similarities.
We can then compare two images through the distributions they induce over the anchors.
In this way, the chosen anchors determine which semantic differences are emphasized by the comparison.

In principle, we could define context-specific similarity directly using a finite anchor set. 
However, this ties the geometry to a particular collection of anchors: it may reflect idiosyncrasies of that finite set, and the cost of representing the geometry grows with the number of anchors.
The main technical insight is to move from identifying the geometry with a finite set of text anchors to identifying it with a probability distribution from which the anchors are drawn i.i.d.
Using ideas from information geometry, we show that each such distribution corresponds to a Bregman divergence defined on the original image embeddings.
This shifts the object specifying semantic geometry from a finite anchor set to an underlying semantic distribution, which can in turn be modeled efficiently.

The development of the paper is:
\begin{enumerate}
	\item First, we show that the KL divergence between softmax distributions over a finite set of semantic anchors has a natural limit as a Bregman divergence on image representations.
	This reveals a direct role of information geometry in organizing semantic similarity in contrastive representations, and leads to \emph{anchor divergences}---a principled method for specifying context-specific semantic geometries.

	\item A key consequence of the Anchor Divergence construction is that the anchor distribution determines the induced Bregman geometry.
	Drawing on classical connections between exponential families and Bregman geometry
	\citep{efron1978geometry,amari1982differential,wainwright2008graphical,amari2010information},
	we use the observed anchors to fit a parametric model of the underlying population anchor distribution, whose fitted distribution induces an explicit Bregman divergence.
	This has two practical benefits: once the model is fit, the cost of representing and evaluating the geometry need not scale with the number of observed anchors, and the parametric structure can provide useful approximations when only a limited number of anchors are available.

	\item As a further consequence of modeling the anchor distribution, we show how broad semantic structure can be adapted to a specific target using an empirical-Bayes-type construction.
	The idea is to first learn a broad anchor distribution---and hence a broad semantic geometry---from a large collection of anchors that are loosely related to the target, e.g., all text about diseases, and then refine it using a small set of targeted anchors.
	Empirically, we find that this can reduce task-specific anchor requirements by a full order of magnitude.

	\item Finally, experiments test the consequences of the framework.
	We show that the anchor distribution controls which semantic factors govern similarity; parametric models are useful when anchors are limited, while the non-parametric softmax improves as anchor coverage grows; and broad semantic structure can be adapted using few target anchors.
	Across these settings, the resulting geometries improve over task-agnostic cosine similarity and task-specific heuristic approaches.
\end{enumerate}

\section{Anchor Divergence}
\label{sec:method}

We begin with a finite set of text anchors $\Gamma=\{\gamma_1,\dots,\gamma_N\}$ chosen to describe the semantics of interest.
For example, in a pathology-specific geometry, the anchors may be clinical descriptions of different findings.
Given an image embedding $\lambda$, the contrastive softmax induces a distribution over these anchors,
$
    \mathbb{P}_\Gamma(\gamma_k\mid\lambda) \propto \exp(\lambda^\top\gamma_k/\tau),
$
where $\tau>0$ is the temperature.
A natural way to compare two image embeddings is then by the KL divergence between the distributions they induce:
\begin{equation*}
    d_{\mathrm{SM},\Gamma}(\lambda_1\|\lambda_2) := \mathrm{KL}\!\left( \mathbb{P}_\Gamma(\cdot\mid\lambda_1)\,\|\,\mathbb{P}_\Gamma(\cdot\mid\lambda_2) \right).
\end{equation*}
We will refer to this as the \emph{softmax anchor divergence}. 

Specifying a divergence via a finite set of anchors is appealingly simple.
However, we may worry that the procedure may be overly sensitive to the specific choice of anchors.
For example, if swapping out a single clinical note substantively altered the geometry of the pathology-specific anchor divergence, it would suggest that we have not captured a meaningful notion of pathology-specific geometry.
This motivates viewing the observed anchors as samples from an underlying distribution over semantically relevant anchors. In this view, the distribution, rather than any particular finite realization, specifies the semantic context. 
We now see that such a distribution induces a well-defined population geometry on the image embedding space.

To state the result, we recall two standard definitions.

\begin{definition}
For differentiable, strictly convex $f$, the \emph{Bregman divergence} is
\begin{equation*}
    B_f(\eta_1\|\eta_2) := f(\eta_1)-f(\eta_2) - \nabla f(\eta_2)^\top(\eta_1-\eta_2).
\end{equation*}
\end{definition}

\begin{definition}
For a distribution $q$ on $\mathbb R^d$, its \emph{cumulant generating function} (CGF) is
\begin{equation*}
    \kappa_q(\eta) := \log \int \exp(\eta^\top\gamma)\,q(d\gamma).
\end{equation*}
\end{definition}

The next result identifies the corresponding population geometry.

\begin{restatable}[Softmax anchor divergence converges to a population Bregman divergence]{theorem}{softmax}
\label{thm:softmax-to-bregman}
Let
$\Gamma_N=\{\gamma_1,\dots,\gamma_N\}$ be sampled i.i.d.\ from an anchor distribution $q$ on $\mathbb S^{d-1}$.
Then, as $N\to\infty$,
\begin{align}
    \sup_{\lambda_1,\lambda_2\in\mathbb S^{d-1}}
    \left|
        d_{\mathrm{SM},\Gamma_N}(\lambda_1\|\lambda_2)
        -
        B_{\kappa_q}
        \left(
            \frac{\lambda_2}{\tau}
            \Big\|
            \frac{\lambda_1}{\tau}
        \right)
    \right|
    \xrightarrow{\mathrm{a.s.}} 0.
    \label{eq:uniform_softmax_convergence}
\end{align}
\end{restatable}

The proof is deferred to \Cref{app:proof}.
The key observation is that the contrastive softmax is an exponential tilt of the empirical measure over anchors.
The resulting natural exponential family links KL divergence to the Bregman geometry generated by $\kappa_q$.

Thus, the simple idea of comparing images through their distributions over semantic anchors yields a population Bregman geometry determined by the anchor distribution. This motivates our central definition.

\begin{definition}
\label{def:anchor_divergence}
For a distribution $q$ on $\mathbb R^d$, the \emph{anchor divergence} is
\begin{align}
    d_q(\lambda_1\|\lambda_2)
    :=
    B_{\kappa_q}
    \left(
        \frac{\lambda_2}{\tau}
        \Big\|
        \frac{\lambda_1}{\tau}
    \right).
\end{align}
\end{definition} 

Under this definition, the anchor distribution $q$ specifies the semantic context and thereby determines the geometry on the image embedding space.
The softmax anchor divergence is the special case induced by the empirical distribution.
This distributional view will allow us to model and adapt the semantic geometry in the following sections.

\section{Modeling Structured Semantic Geometries}
\label{sec:parametric_models}

We have arrived at anchor divergences as Bregman divergences induced by an anchor \emph{distribution} $q$ that represents the semantics of interest.
In practice, however, we observe only a finite set of anchors.
We therefore turn to modeling the underlying distribution from these observations, and hence the induced semantic geometry.

The simplest approach is to use the empirical measure on the anchor set. It is straightforward to verify that the anchor divergence associated with this empirical distribution is exactly the softmax anchor divergence defined earlier. 
As established by \Cref{thm:softmax-to-bregman} this will converge to the target population geometry as we get more and more anchors.
With limited anchors, however, the finite anchor set may poorly represent the underlying semantics.

Since $q$ determines the induced geometry, we can impose structure on the geometry by modeling $q$.
The idea is to specify a parametric model $q_\theta$ for the anchor distribution. For example, we may take $q_\theta$ to be a multivariate Gaussian distribution with $\theta$ its mean and covariance. We then fit $\hat{\theta}$ to the specified anchors, say by maximum likelihood.
The fitted distribution $q_{\hat{\theta}}$ then induces the corresponding anchor divergence through \Cref{def:anchor_divergence}.
Thus, different anchor-distribution assumptions induce different structured geometries in image embedding space.

This reflects the usual tradeoff of parametric modeling. Simpler models impose stronger assumptions, which can help when anchors only partially specify the target semantics, but may miss the target geometry if those assumptions are too restrictive, even with many anchors. Moreover, once fit, the divergence is evaluated through the fitted distribution rather than all $N$ observed anchors, so its evaluation cost need not scale with the anchor-set size; we examine this empirically in \Cref{app:runtime}.

\begin{table}[t]
\caption{
\textbf{Modeling the anchor distribution induces different semantic geometries.}
Each column specifies an anchor-distribution model and its CGF.
The empirical model places uniform mass on the observed anchors; for the Gaussian model, $\hat{\mu}$ and $\hat{\Sigma}$ are the sample mean and covariance; for the Gaussian mixture, $\{(\hat{\pi}_m,\hat{\mu}_m,\hat{\Sigma}_m)\}_{m=1}^M$ are fit by expectation--maximization.
}
\label{tab:parametric_models}

\begin{center}

\small
\setlength{\tabcolsep}{5pt}
\renewcommand{\arraystretch}{1.7}

\begin{tabular}{cccc}
\toprule

\textbf{Anchor model}
& Empirical
& Gaussian
& Gaussian Mixture
\\

\midrule

CGF $\kappa(\eta)$
&
$\displaystyle
\log \frac{1}{|\Gamma|}
\sum_{\gamma\in\Gamma}
\exp(\eta^\top\gamma)$
&
$\displaystyle
\hat{\mu}^\top\eta
+\frac{1}{2}\eta^\top\hat{\Sigma}\eta$
&
$\displaystyle
\log\sum_{m=1}^{M}
\hat{\pi}_m
\exp\left(
    \hat{\mu}_m^\top\eta
    +\frac{1}{2}\eta^\top\hat{\Sigma}_m\eta
\right)$
\\

\bottomrule
\end{tabular}

\end{center}
\end{table}

\Cref{tab:parametric_models} illustrates this perspective for three anchor models.
The empirical distribution recovers softmax anchor divergence.
The Gaussian model is the simplest parametric alternative: fitting its mean and covariance yields a quadratic CGF, and hence an elliptical geometry.
The resulting geometry emphasizes directions along which the modeled anchor distribution varies.
A Gaussian mixture allows richer structure: its CGF is a log-sum-exp of quadratic potentials, allowing the induced geometry to reflect multiple semantic modes rather than a single global covariance structure.

\paragraph{Bregman divergence in a projected anchor space.}
We also consider mildly generalizing from Bregman divergences on the contrastive embedding space to Bregman divergences on a subspace of the embedding space. 
This is motivated by evidence that semantic structure is often encoded in
low-dimensional subspaces \citep{oldfield2023parts,bhalla2024interpreting, engels2025not,park2025geometry}. 
We learn a projection, optionally followed by normalization, onto a semantically
relevant subspace from the anchor set and define the Anchor Divergence on the
projected space. 
In \Cref{app:latent_expfam}, we show that this corresponds to a latent-variable model of the anchor distribution and yields a procedure for estimating the projector. 
We refer to the resulting divergences as \emph{latent Gaussian} and \emph{latent spherical} anchor divergences. 
Operationally, both project onto the top-$k$ eigenvectors of the empirical anchor covariance; the latent Gaussian divergence then uses Euclidean distance, while the latent spherical divergence uses cosine distance.
\footnote{We note the latent spherical case yields the same projected cosine structure proposed in \citep{herzog2026reevaluating}} 

The empirical and parametric constructions provide different ways of representing the anchor distribution, and hence different approximations to the same target semantic geometry.
\Cref{fig:chexpert_qualitative} provides a qualitative illustration of the usefulness of the anchor divergence framework and its parametric approximations: both emphasize the pathology of interest over nuisance device-related structure.

\begin{figure}[t]
    \centering
    \includegraphics[
        width=0.88\linewidth
    ]{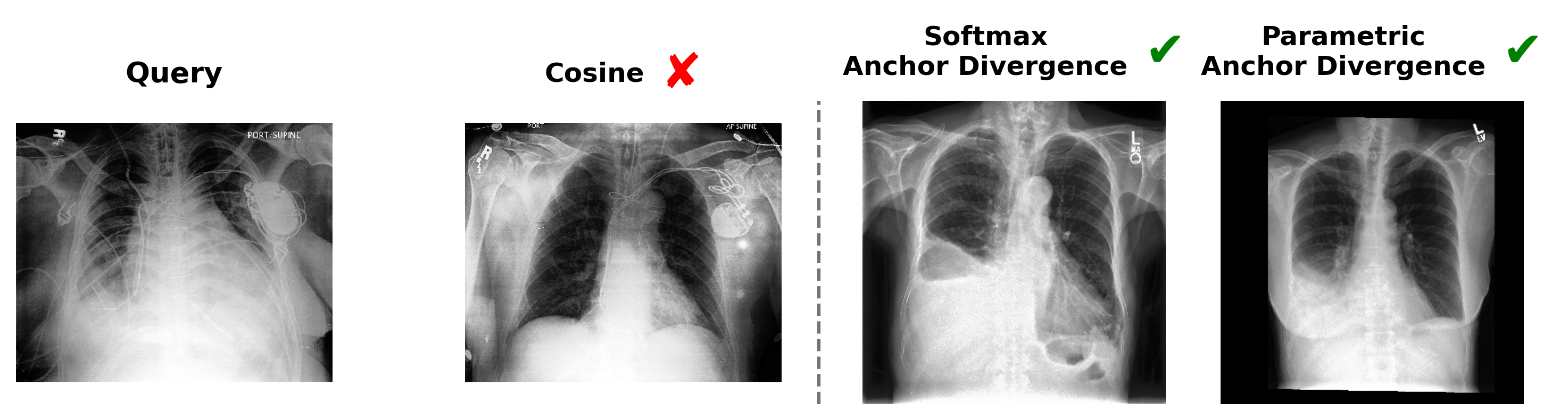}
	\caption{\textbf{Anchor divergences can emphasize relevant semantics over nuisance variation.}
	The query X-ray shows the target pathology, pleural effusion, as well as a support device.
	Cosine similarity retrieves an image sharing device-related structure but not the target pathology, whereas both softmax and (Gaussian model) parametric anchor divergences retrieve images with pleural effusion.}
    \label{fig:chexpert_qualitative}
\end{figure}

\section{Adapting Semantic Geometries from Broad Anchors}
\label{sec:adaptation}

Obtaining many target-specific anchors can be cumbersome, while broader related anchors may be easier to obtain.
The distributional view enables an empirical-Bayes approach: broad anchors provide prior information about the parameters of the anchor distribution, which is then updated using scarce target anchors.
For example, in pathology-specific retrieval, we may have abundant general clinical text but only a few precise descriptions of the target pathology.
The goal is to specialize broad semantic structure using a small number of pathology-specific anchors.

Let $\Gamma_{\mathrm{broad}}$ and $\Gamma_{\mathrm{target}}$ denote the broad and target anchor sets, respectively. Let $\{q_\theta:\theta\in\Theta\}$ be a parametric anchor model with prior $p(\theta\mid\Omega)$. We adapt the anchor distribution in three steps:
\begin{enumerate}
    \item fit $q_\theta$ to $\Gamma_{\mathrm{broad}}$ and use the fit to set the prior hyperparameters $\Omega$;

    \item update this prior using $\Gamma_{\mathrm{target}}$ via a model-specific posterior or posterior-like update;

	\item summarize the update by $\hat{\theta}$, such as a posterior mean or MAP estimate.
\end{enumerate}

The adapted distribution $q_{\hat{\theta}}$ induces the corresponding semantic geometry through its anchor divergence, giving the following definition.

\begin{definition}
\label{def:posterior_anchor_divergence}

Given $\Gamma_{\mathrm{broad}}$ and $\Gamma_{\mathrm{target}}$, let $q_{\hat{\theta}}$ denote the adapted anchor distribution obtained by using $\Gamma_{\mathrm{broad}}$ to define prior information and updating it with $\Gamma_{\mathrm{target}}$. The \emph{posterior anchor divergence} is the anchor divergence induced by $q_{\hat{\theta}}$.

\end{definition}

All methods target the same object: a Bregman geometry induced by an anchor distribution. They differ only in how that distribution is estimated. Softmax anchor divergence uses the empirical distribution over target anchors; parametric anchor divergence fits a structured model to those anchors; posterior anchor divergence further regularizes this structured estimate using broad, related anchors as prior information. When $|\Gamma_{\mathrm{target}}|$ is small, broad anchors stabilize the geometry; as $|\Gamma_{\mathrm{target}}|$ grows, target anchors increasingly determine it. Model-specific updates are in \Cref{app:adaptation_appendix}.

\section{Experiments}
\label{sec:experiments}

Our experiments are designed to test the geometric and statistical predictions of the framework. We organize the experiments around the following questions:
\begin{enumerate}
    \item Can Anchor Divergences effectively specialize geometry to different semantic factors?
    \item With limited anchors, can parametric models usefully approximate the target geometry?
    \item Does posterior anchor modeling enable semantic adaptation using few target anchors?
\end{enumerate}

\subsection{Implementation details}
\label{sec:common_setup}

All natural-image experiments \citep{russakovsky2015imagenet,hendrycks2021many,sangkloy2016sketchy,wah2011caltech} use MetaCLIP-2 \citep{chuang2025metaclip2}.
Class-name anchors use the prompt template ``a photo of a \{class\}'', and style anchors use ``a \{style\} rendition of an object''.
For the clinical experiments, we use CheXpert \citep{irvin2019chexpert} with CheXzero \citep{tiu2022expert}. 
Image and text embeddings are normalized to unit $\ell_2$ norm, and we fix the inverse temperature $1/\tau$ at the model-provided logit scale.

For asymmetric divergences, we rank keys using the symmetrized score
\begin{equation*}
    d_{\mathrm{sym}}(\lambda_q,\lambda_k)
    =
    \frac{1}{2}
    \left[
        d(\lambda_q\|\lambda_k)
        +
        d(\lambda_k\|\lambda_q)
    \right].
\end{equation*}
We report Area Under the Precision--Recall Curve (AUPRC), which is suitable when relevant keys are sparse and prevalence varies across semantic targets.
Relevance is defined by the semantic labels of each experiment.

We compare against two reference geometries.
Cosine similarity provides the standard task-agnostic geometry.
Image-text inversion (ITI) provides a task-specific heuristic baseline using a predict-then-retrieve procedure: it first matches the query image to the closest anchor and then ranks keys by cosine similarity to this anchor.
This follows the spirit of image inversion methods in \citet{baldrati2023zero,mistretta2025cross}.
See \Cref{app:implementation} for full implementation details.

\subsection{Anchor distributions select the semantic coordinate}
\label{sec:exp_anchor_control}
We first isolate the effect of the anchor distribution with the embeddings fixed. \Cref{fig:intro_semantic_neighborhoods} gives a qualitative illustration. We test this systematically: if the anchor distribution determines the semantic geometry, then changing it should change which semantic coordinate is preserved by retrieval.

\paragraph{Clinical coordinate control.}

\begin{table}[t]
\caption{
\textbf{Pathology-specific anchors specialize clinical retrieval.}
Condition-wise AUPRC differences between softmax anchor divergence and each reference geometry; positive values favor anchor divergence.
Results are averaged over 10 independently generated anchor sets and 10 CheXzero checkpoints, with 95\% intervals accounting for bootstrap and anchor-generation variability.
Values $>0.005$ are bold, values in $[-0.005,0.005]$ are gray, and values $<-0.005$ are bold red.
}
\label{tab:chexpert_softmax_diffs}

\begin{center}
\small
\setlength{\tabcolsep}{16pt}
\renewcommand{\arraystretch}{0.98}

\begin{tabular}{lc@{\hspace{8em}}c}
\toprule

\textbf{Pathology}
&
\textbf{Softmax $-$ Cosine}
&
\textbf{Softmax $-$ ITI}
\\

\midrule

Enlarged Cardiomediastinum
& \neu{$-0.0002 \pm 0.0003$}
& \neu{$-0.0045 \pm 0.0010$}
\\

Cardiomegaly
& \pos{$0.0279 \pm 0.0034$}
& \pos{$0.0069 \pm 0.0091$}
\\

Lung Opacity
& \pos{$0.0090 \pm 0.0032$}
& \pos{$0.0137 \pm 0.0035$}
\\

Lung Lesion
& \neu{$0.0001 \pm 0.0005$}
& \pos{$0.0051 \pm 0.0010$}
\\

Edema
& \pos{$0.0316 \pm 0.0044$}
& \pos{$0.0286 \pm 0.0072$}
\\

Consolidation
& \pos{$0.0062 \pm 0.0014$}
& \pos{$0.0073 \pm 0.0023$}
\\

Pneumonia
& \pos{$0.0081 \pm 0.0013$}
& \pos{$0.0117 \pm 0.0018$}
\\

Atelectasis
& \pos{$0.0095 \pm 0.0027$}
& \pos{$0.0113 \pm 0.0036$}
\\

Pneumothorax
& \pos{$0.0080 \pm 0.0013$}
& \pos{$0.0128 \pm 0.0026$}
\\

Pleural Effusion
& \pos{$0.0465 \pm 0.0060$}
& \pos{$0.0327 \pm 0.0098$}
\\

Pleural Other
& \neu{$-0.0009 \pm 0.0003$}
& \neu{$0.0025 \pm 0.0005$}
\\

Fracture
& \neu{$0.0003 \pm 0.0002$}
& \neu{$0.0030 \pm 0.0005$}
\\

Support Devices
& \pos{$0.0208 \pm 0.0056$}
& \pos{$0.0366 \pm 0.0105$}
\\

\midrule

\textbf{Average}
& \pos{$0.0128 \pm 0.0011$}
& \pos{$0.0129 \pm 0.0014$}
\\

\bottomrule
\end{tabular}

\end{center}
\end{table}

For each CheXpert finding, we independently generate 10 sets of report-style clinical text anchors from a fixed LLM prompt; see \Cref{app:chexpert_anchors}. We evaluate retrieval using the corresponding pathology label and average over anchor generations and the 10 released CheXzero checkpoints. \Cref{tab:chexpert_softmax_diffs} reports AUPRC differences relative to cosine similarity and ITI, with 95\% intervals accounting for anchor-generation and bootstrapping.
Anchor divergence improves over the reference geometries for most findings. The largest gains occur for pleural effusion, cardiomegaly, edema, and support devices. These results indicate that pathology-specific anchors can shift retrieval toward the intended clinical coordinate beyond generic image-space similarity.

\paragraph{Object--style control.}

We next test whether varying the anchor distribution can interpolate between two semantic coordinates. ImageNet-R contains ImageNet object classes rendered in diverse styles, allowing retrieval to be evaluated by either object identity or style. Let $q_{\mathrm{obj}}$ and $q_{\mathrm{style}}$ denote the corresponding empirical anchor distributions, and mix them as $q_p = p q_{\mathrm{obj}} + (1-p) q_{\mathrm{style}}$, with $p\in[0,1]$. We evaluate retrieval as $p$ varies, averaging over 10 random realizations of the anchor set. \Cref{fig:anchor_control} shows that increasing object-anchor mass increases object retrieval while decreasing style retrieval.

\begin{figure}[t]
    \centering
    \includegraphics[width=0.98\linewidth]{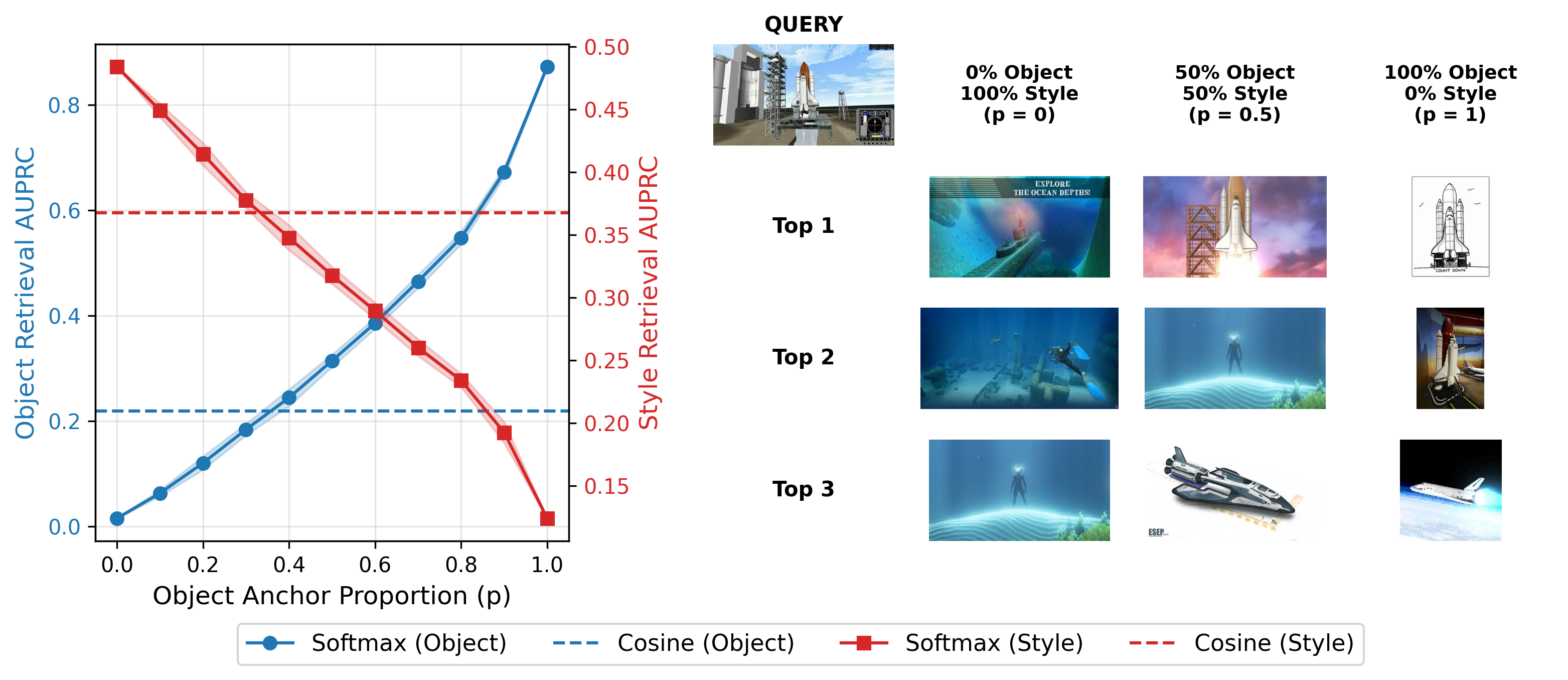}
    \caption{
	\textbf{Anchor distributions select the semantic coordinate.}
	\textbf{Left:} Mixing object and style anchor distributions on ImageNet-R shifts retrieval from preserving style to preserving object identity as $p$ increases.
	\textbf{Right:} For a video-game-style space shuttle query, retrieved neighbors similarly shift from style-matched to object-matched images.
    }
    \label{fig:anchor_control}
\end{figure}

Together, these experiments show that the same fixed image embeddings support different semantic geometries solely through changes to the anchor distribution.

\subsection{Parametric modeling is useful with limited anchors}
\label{sec:exp_estimator_tradeoff}

We next test whether parametric models of the anchor distribution provide useful approximations to the target geometry when only limited anchors are available. Softmax uses the empirical anchor distribution directly, whereas parametric models impose additional structure. We therefore expect parametric models to be most useful when anchors are limited, while the empirical construction should improve as anchor coverage grows.

We first vary the number of ImageNet class anchors and evaluate retrieval on ImageNet. \Cref{fig:q2_estimator_tradeoff} left shows the intuitive sample-efficiency pattern: structured geometries recover useful semantic neighborhoods at moderate anchor budgets, while softmax improves with anchor coverage and eventually becomes strongest. Similar trends hold across the other datasets; see \Cref{app:additional_experiments}.

We then consider the high-coverage regime using full dataset-matched anchor sets on ImageNet-R, CUB-200, and Sketchy. \Cref{fig:q2_estimator_tradeoff} right shows that softmax anchor divergence is strongest in this setting, while parametric anchor divergences retain much of its gain over cosine similarity. Together with the ImageNet results, this illustrates the modeling tradeoff from \Cref{sec:parametric_models}: parametric structure can be useful with limited anchors, while the empirical anchor distribution becomes increasingly effective as anchor coverage grows. 

\begin{figure}[t]
    \centering
    \includegraphics[width=\linewidth]{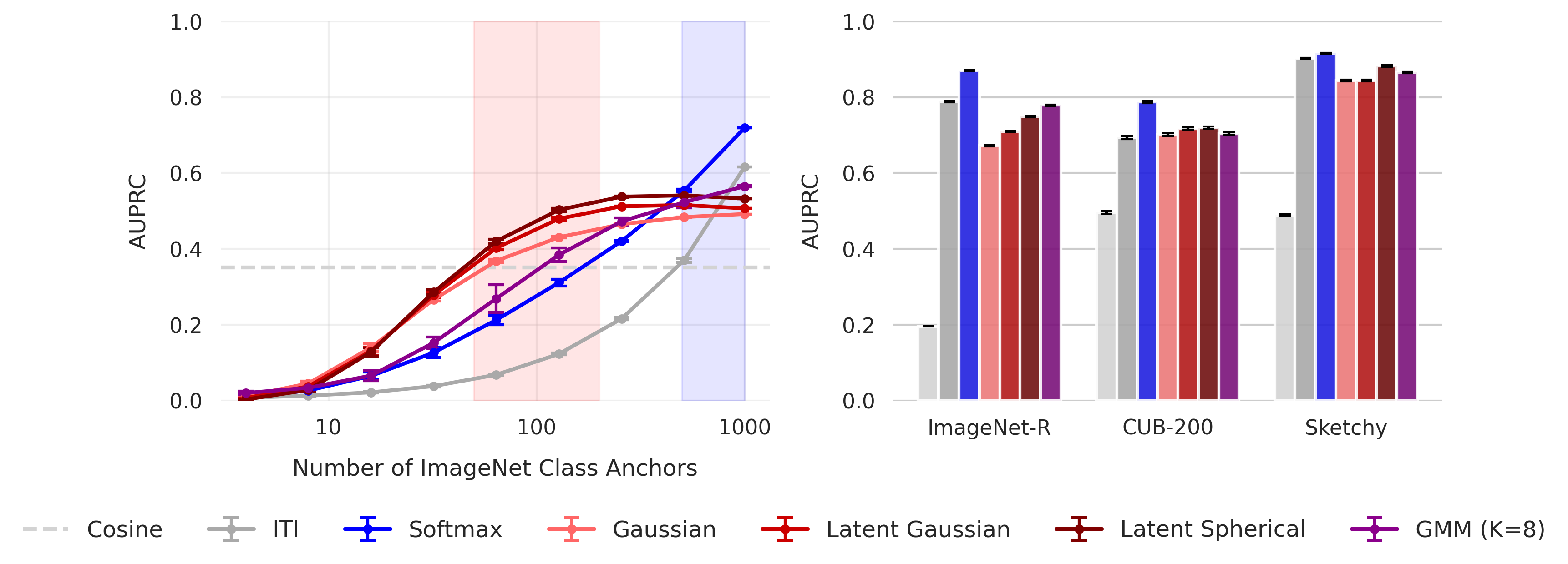}

	\caption{
	\textbf{Parametric modeling is useful with limited anchors.}
	\textbf{Left:} On ImageNet, parametric anchor divergences are strongest in the red-shaded limited-anchor regime, while softmax improves with coverage and becomes strongest in the blue high-coverage regime.
	\textbf{Right:} With full dataset-matched anchor sets on ImageNet-R, CUB-200, and Sketchy, softmax gives the strongest retrieval performance, while parametric anchor divergences retain much of its gain over cosine similarity.
	}
    \label{fig:q2_estimator_tradeoff}
\end{figure}

\subsection{Posterior modeling adapts broad geometry with few target anchors}
\label{sec:exp_semantic_adaptation}

Finally, we test whether broad semantic structure can be adapted toward a target geometry using only a few task-specific anchors. This setting exposes a limitation of softmax anchor divergence: as a nonparametric estimator, it places mass only on observed target anchors, so with few anchors it can yield a poor approximation of the target geometry. Posterior anchor models instead use broad, semantically related anchors as prior information and specialize this structure using target anchors.

We evaluate this on ImageNet, using WordNet anchors as broad semantic structure and ImageNet class names as target-specific anchors. We vary the number of target anchors and compare target-only softmax, latent spherical, and GMM anchor divergences with posterior latent spherical and posterior GMM models adapted from WordNet, alongside cosine similarity. Results are averaged over five random seeds; model-specific adaptation details are given in \Cref{app:adaptation_appendix}.

\Cref{fig:semantic_adaptation} shows that posterior anchor models are most useful when target anchors are scarce. Broad anchors provide prior semantic structure, while target anchors specialize the geometry toward the ImageNet class coordinate. As target-anchor coverage grows, the gap to target-only models narrows as the target anchors increasingly determine the geometry.

\paragraph{Summary.}
The experiments support the main consequences of the framework. Changing the anchor distribution changes which semantic factors govern similarity; parametric models provide useful approximations when anchors are limited, while the empirical softmax geometry improves with anchor coverage; and posterior anchor models adapt broad semantic structure using few task-specific anchors. Together, these results support anchor divergences as a flexible framework for specifying, modeling, and adapting context-dependent semantic geometry in contrastive representation spaces.

\begin{figure}[t]
    \centering
    \includegraphics[width=0.8\linewidth]{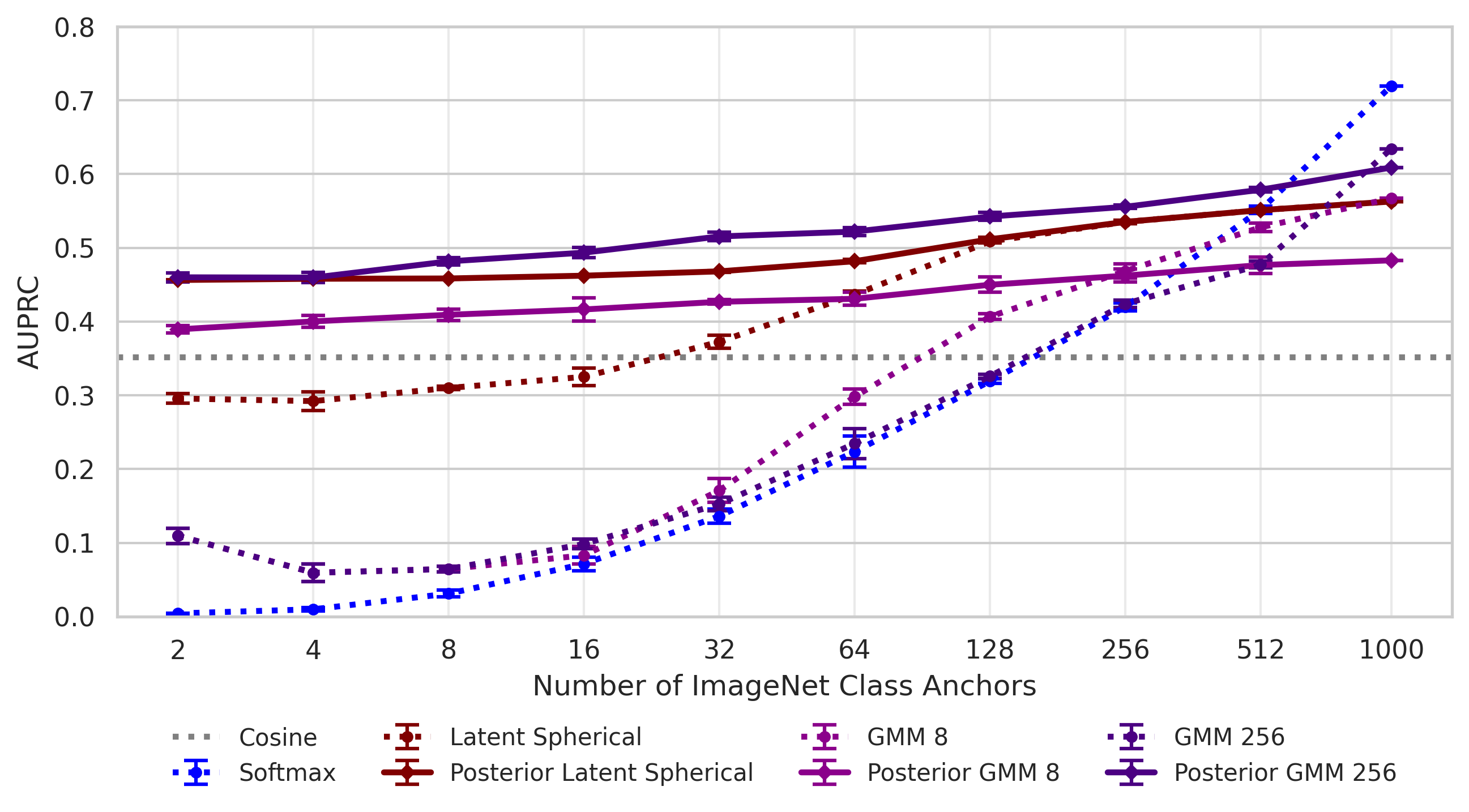}

	\caption{
	\textbf{Posterior anchor modeling adapts broad geometries using few target anchors.}
	Dotted curves use only target-specific ImageNet anchors, while solid curves use posterior anchor models adapted from WordNet with the same target anchors.
	Posterior models are strongest at small and moderate target-anchor budgets, with the advantage narrowing as coverage grows.
	}

    \label{fig:semantic_adaptation}
\end{figure}

\section{Discussion and Related Work}

\paragraph{Semantic geometry in contrastive representations.}
Prior work has shown that learned representations often encode semantic structure geometrically, from word-embedding analogies and latent-variable accounts of embeddings to linear structure in neural representations
\citep{mikolov2013efficient,arora2016latent,saxe2019mathematical,park2023linear,park2025geometry}.
In vision--language models, related work has studied how CLIP-like representations organize concepts through directions, parts, subspaces, and interpretable coordinates
\citep{radford2021learning,chuang2025metaclip2,oldfield2023parts,bhalla2024interpreting,gandelsman2024interpreting,merullo2022linearly}.
Our question is complementary: given that a representation may encode many semantic factors simultaneously, which of them should determine similarity in a particular context?
Anchor divergences answer this by representing the semantic context through an anchor distribution, which induces a corresponding Bregman geometry on the fixed representation space.

\paragraph{Image--image comparison in CLIP-like models.}
CLIP-style image embeddings are commonly compared using cosine similarity across applications
\citep{abbasi2025clip,zheng2025benchmarking,denner2025leveraging,tiu2022expert}.
Several works have proposed alternatives to raw image--image cosine comparison, including cache-based adaptation, cross-modal aggregation, text-based inversion, and methods addressing image--text or intra-modal mismatch
\citep{zhang2022tip,udandarao2023sus,yi2024leveraging,mistretta2025cross,baldrati2023zero,herzog2026reevaluating}. 
These works motivate alternatives to the fixed cosine geometry, but do not generally ask how semantic context itself should determine the geometry on a fixed representation space.

\paragraph{Relation to TIP-X.}
TIP-X uses text as a bridge for image--image comparison: each image is mapped to a softmax affinity profile over downstream class-text prompts, and images are compared by the KL divergence between these profiles \citep{udandarao2023sus}.
This is the same finite-anchor construction as our softmax anchor divergence when the anchors are the downstream class-text prompts.
Our contribution begins from this finite construction and studies the geometry determined by an underlying distribution over anchors.
We show that an anchor distribution induces a population Bregman geometry and use this perspective to model and adapt semantic geometry.

\paragraph{Information geometry and metric learning.}
Our construction builds on classical connections between exponential families, KL divergence, and Bregman geometry
\citep{efron1978geometry,amari1982differential,banerjee2005clustering,wainwright2008graphical,amari2010information}.
It is also related to metric learning, which learns task-specific distances from supervised constraints on the representation space
\citep{xing2002distance,weinberger2009distance,schroff2015facenet,movshovitz2017no,musgrave2020metric,roth2020revisiting}.
Anchor divergences instead specify the semantic context through a distribution over semantic text anchors, without requiring labeled image pairs or representation fine-tuning.
Because this anchor distribution determines the induced geometry, parametric modeling imposes corresponding geometric structure, while Bayesian updating provides a mechanism for semantic adaptation.

\paragraph{Future directions.}
Although our experiments focus on image--image retrieval in vision--language models, the construction is not intrinsically tied to images. The key ingredient is that a representation in one space defines a softmax distribution over anchors in another space through query--key scores. The theoretical structure developed here follows from this softmax, or equivalently exponential-family, form rather than from any image-specific property. In this view, a distribution over keys induces a geometry on the query space. This suggests that anchor divergences extend naturally to other contrastive and multimodal settings with the same structure. More broadly, this points to the importance of understanding the geometry of softmax-based models as a general problem in AI: recent work argues that representations defining softmax distributions are naturally studied through information geometry, with implications for probing and steering \citep{park2026information}. Progress on anchor divergences may therefore be of independent interest beyond retrieval, contributing to a broader theory of semantic geometry in softmax models.

\subsubsection*{Acknowledgments}
This work is supported by ONR grant N00014-23-1-2591.

\printbibliography

\appendix
\crefalias{section}{appendix}
\crefalias{subsection}{appendix}

\section{Proof of \Cref{thm:softmax-to-bregman}}\label{app:proof}

\softmax*
\begin{proof}
	We begin by introducing some notation and definitions.
	We define
	\begin{align*}
		\eta=\frac{\lambda}{\tau}, \qquad \mathcal E:=\left\{\frac{\lambda}{\tau}:\lambda\in \mathcal S^{d-1}\right\} \subset \mathbb R^d.
	\end{align*}
	Since $\mathcal S^{d-1}$ is compact, $\mathcal E$ is compact as well.
	We rewrite the population cumulant generating function as
	\begin{align*}
		\kappa(\eta) = \log \int e^{\eta^\top \gamma}\,q(d\gamma) = \log \mathbb E_q\left[e^{\eta^\top\gamma}\right].
	\end{align*}
	For an anchor set $\Gamma_N=\{\gamma_1,\dots,\gamma_N\}$, the cumulant generating function of the empirical measure is 
	\begin{align*}
		\widehat\kappa_N(\eta) := \log\left(\frac1N\sum_{i=1}^N e^{\eta^\top \gamma_i}\right).
	\end{align*}
	The factor $1/N$ does not affect the induced softmax probabilities, but it is convenient for the law of large numbers.
	We also define
	\begin{align*}
		Z_N(\eta):=\frac1N\sum_{i=1}^N e^{\eta^\top \gamma_i}, \qquad Z(\eta):=\mathbb E_q[e^{\eta^\top \gamma}],
	\end{align*}
	and
	\begin{align*}
		M_N(\eta):=\frac1N\sum_{i=1}^N \gamma_i e^{\eta^\top \gamma_i}, \qquad M(\eta):=\mathbb E_q[\gamma e^{\eta^\top \gamma}].
	\end{align*}
	Throughout this proof, the anchors are sampled i.i.d.\ from $q$.

	Now we are ready to prove the theorem. The first step is to rewrite the discrete softmax family as an empirical exponential family. Given $\Gamma_N=\{\gamma_1,\dots,\gamma_N\}$ and $\eta\in\mathbb R^d$,
	\begin{align*}
		p_N(i\mid \eta) &= \frac{e^{\eta^\top \gamma_i}}{\sum_{j=1}^N e^{\eta^\top \gamma_j}} \\
		&= \frac{\frac1N e^{\eta^\top \gamma_i}}{\frac1N\sum_{j=1}^N e^{\eta^\top \gamma_j}} \\
		&= \exp\big(\eta^\top \gamma_i-\widehat\kappa_N(\eta)\big)\cdot \frac1N.
	\end{align*}
	Thus, the discrete softmax is exactly the exponential tilt of the empirical anchor measure.

	Let $\eta_a=\lambda_a/\tau$ for $a=1,2$. Then
	\begin{align*}
		d_{\mathrm{SM}}(\lambda_1\|\lambda_2\mid \Gamma_N) &= \sum_{i=1}^N p_N(i\mid \eta_1) \log \frac{p_N(i\mid \eta_1)}{p_N(i\mid \eta_2)} \\
		&= \sum_{i=1}^N p_N(i\mid \eta_1) \Big((\eta_1-\eta_2)^\top \gamma_i -\widehat\kappa_N(\eta_1) +\widehat\kappa_N(\eta_2)\Big).
	\end{align*}
	Since $\sum_i p_N(i\mid \eta_1)=1$, this becomes
	\begin{align}
		d_{\mathrm{SM}}(\lambda_1\|\lambda_2\mid \Gamma_N) &= (\eta_1-\eta_2)^\top \sum_{i=1}^N p_N(i\mid \eta_1)\gamma_i -\widehat\kappa_N(\eta_1)+\widehat\kappa_N(\eta_2). \label{eq:sm-kl-expand-app}
	\end{align}
	But
	\begin{align*}
		\nabla \widehat\kappa_N(\eta) = \frac{\sum_{i=1}^N \gamma_i e^{\eta^\top \gamma_i}} {\sum_{j=1}^N e^{\eta^\top \gamma_j}} = \sum_{i=1}^N p_N(i\mid \eta)\gamma_i,
	\end{align*}
	so \Cref{eq:sm-kl-expand-app} can be written as
	\begin{align}
		d_{\mathrm{SM}}(\lambda_1\|\lambda_2\mid \Gamma_N) &= \widehat\kappa_N(\eta_2)-\widehat\kappa_N(\eta_1) +(\eta_1-\eta_2)^\top \nabla \widehat\kappa_N(\eta_1). \label{eq:empirical-bregman-eta1-app}
	\end{align}
	Equivalently,
	\begin{align}
		d_{\mathrm{SM}}(\lambda_1\|\lambda_2\mid \Gamma_N) = B_{\widehat\kappa_N}(\eta_2\|\eta_1), \label{eq:empirical-bregman-app}
	\end{align}
	where $B_f(y\|x):=f(y)-f(x)-\nabla f(x)^\top (y-x)$.

	We now show uniform convergence of $Z_N$ to $Z$ on $\mathcal E$. Since $\|\eta\|=1/\tau$ for $\eta\in\mathcal E$ and $\|\gamma\|=1$, we have
	\begin{align*}
		0< e^{\eta^\top\gamma}\le e^{1/\tau} \qquad \text{for all } \eta\in\mathcal E,\ \gamma\in \mathcal S^{d-1}.
	\end{align*}
	Moreover, for any $\eta,\eta'\in\mathcal E$ and $\gamma\in\mathcal S^{d-1}$,
	\begin{align*}
		\left|e^{\eta^\top\gamma}-e^{(\eta')^\top\gamma}\right| &\le e^{1/\tau}\,|(\eta-\eta')^\top\gamma| \le e^{1/\tau}\|\eta-\eta'\|.
	\end{align*}
	Therefore, both $Z_N$ and $Z$ are Lipschitz on $\mathcal E$ with deterministic Lipschitz constant at most $e^{1/\tau}$:
	\begin{align*}
		|Z_N(\eta)-Z_N(\eta')| \le e^{1/\tau}\|\eta-\eta'\|, \qquad |Z(\eta)-Z(\eta')| \le e^{1/\tau}\|\eta-\eta'\|.
	\end{align*}

	Fix $\varepsilon>0$, and let $T_\varepsilon\subset \mathcal E$ be a finite $\varepsilon$-net of the compact set $\mathcal E$. For each fixed $\eta\in T_\varepsilon$, the strong law of large numbers gives
	\begin{align*}
		Z_N(\eta)\xrightarrow{\mathrm{a.s.}} Z(\eta).
	\end{align*}
	Since $T_\varepsilon$ is finite, it follows that
	\begin{align*}
		\max_{\eta\in T_\varepsilon}|Z_N(\eta)-Z(\eta)|\xrightarrow{\mathrm{a.s.}}0.
	\end{align*}
	Now let $\eta\in\mathcal E$, and choose $\eta_\varepsilon\in T_\varepsilon$ with $\|\eta-\eta_\varepsilon\|\le \varepsilon$. Then
	\begin{align*}
		|Z_N(\eta)-Z(\eta)| &\le |Z_N(\eta)-Z_N(\eta_\varepsilon)| + |Z_N(\eta_\varepsilon)-Z(\eta_\varepsilon)| + |Z(\eta_\varepsilon)-Z(\eta)| \\
		&\le 2e^{1/\tau}\varepsilon + \max_{\eta'\in T_\varepsilon}|Z_N(\eta')-Z(\eta')|.
	\end{align*}
	Taking the supremum over $\eta\in\mathcal E$ yields
	\begin{align*}
		\sup_{\eta\in\mathcal E}|Z_N(\eta)-Z(\eta)| \le 2e^{1/\tau}\varepsilon + \max_{\eta'\in T_\varepsilon}|Z_N(\eta')-Z(\eta')|.
	\end{align*}
	Hence, almost surely,
	\begin{align*}
		\limsup_{N\to\infty}\sup_{\eta\in\mathcal E}|Z_N(\eta)-Z(\eta)| \le 2e^{1/\tau}\varepsilon.
	\end{align*}
	Since $\varepsilon>0$ is arbitrary, we conclude
	\begin{align}
		\sup_{\eta\in\mathcal E}|Z_N(\eta)-Z(\eta)|\xrightarrow{\mathrm{a.s.}}0. \label{eq:uniform-Z-app}
	\end{align}

	Using the same Lipschitz argument, for any $\eta,\eta'\in\mathcal E$ and $\gamma\in\mathcal S^{d-1}$,
	\begin{align*}
		\|\gamma e^{\eta^\top\gamma}-\gamma e^{(\eta')^\top\gamma}\| \le e^{1/\tau}\|\eta-\eta'\|.
	\end{align*}
	Therefore, both $M_N$ and $M$ are Lipschitz on $\mathcal E$ with deterministic Lipschitz constant at most $e^{1/\tau}$:
	\begin{align*}
		\|M_N(\eta)-M_N(\eta')\| \le e^{1/\tau}\|\eta-\eta'\|, \qquad \|M(\eta)-M(\eta')\| \le e^{1/\tau}\|\eta-\eta'\|.
	\end{align*}
	Applying the same finite-net argument componentwise gives
	\begin{align}
		\sup_{\eta\in\mathcal E}\|M_N(\eta)-M(\eta)\|\xrightarrow{\mathrm{a.s.}}0. \label{eq:uniform-M-app}
	\end{align}

	Since
	\begin{align*}
		\|\gamma e^{\eta^\top\gamma}\|\le e^{1/\tau} \qquad \text{for all }\eta\in\mathcal E,\ \gamma\in\mathcal S^{d-1},
	\end{align*}
	dominated convergence justifies differentiation under the expectation, so
	\begin{align*}
		\nabla Z(\eta)=\mathbb E_q[\gamma e^{\eta^\top\gamma}]=M(\eta).
	\end{align*}
	Hence,
	\begin{align*}
		\nabla \kappa(\eta)=\frac{\nabla Z(\eta)}{Z(\eta)}=\frac{M(\eta)}{Z(\eta)}, \qquad \nabla \widehat\kappa_N(\eta)=\frac{M_N(\eta)}{Z_N(\eta)}.
	\end{align*}

	By \cref{eq:uniform-Z-app}, $Z_N\to Z$ uniformly almost surely on $\mathcal E$. Also, $Z(\eta)>0$ for every $\eta$, and $Z$ is continuous on compact $\mathcal E$, so
	\begin{align*}
		m:=\inf_{\eta\in\mathcal E} Z(\eta)>0.
	\end{align*}
	Therefore, almost surely for all sufficiently large $N$,
	\begin{align*}
		\inf_{\eta\in\mathcal E} Z_N(\eta)\ge \frac m2>0.
	\end{align*}
	Because $\log$ is Lipschitz on $[m/2,\infty)$, \Cref{eq:uniform-Z-app} implies
	\begin{align}
		\sup_{\eta\in\mathcal E}|\widehat\kappa_N(\eta)-\kappa(\eta)| \xrightarrow{\mathrm{a.s.}}0. \label{eq:uniform-kappa-app}
	\end{align}

	Using \Cref{eq:uniform-Z-app} and \Cref{eq:uniform-M-app}, together with the lower bound on the denominators, we obtain
	\begin{align}
		\sup_{\eta\in\mathcal E}\|\nabla \widehat\kappa_N(\eta)-\nabla\kappa(\eta)\| \xrightarrow{\mathrm{a.s.}}0. \label{eq:uniform-grad-app}
	\end{align}

	Recall that
	\begin{align*}
		B_\kappa(\eta_2\|\eta_1) = \kappa(\eta_2)-\kappa(\eta_1)-\nabla\kappa(\eta_1)^\top(\eta_2-\eta_1) = \kappa(\eta_2)-\kappa(\eta_1)+(\eta_1-\eta_2)^\top \nabla\kappa(\eta_1).
	\end{align*}
	Hence,
	\begin{align*}
		&\sup_{\eta_1,\eta_2\in\mathcal E} \left| d_{\mathrm{SM}}(\tau\eta_1\|\tau\eta_2\mid\Gamma_N) - B_\kappa(\eta_2\|\eta_1) \right| \\
		&\le 2\sup_{\eta\in\mathcal E}|\widehat\kappa_N(\eta)-\kappa(\eta)| + \sup_{\eta_1,\eta_2\in\mathcal E} \left|(\eta_1-\eta_2)^\top \big(\nabla\widehat\kappa_N(\eta_1)-\nabla\kappa(\eta_1)\big)\right|.
	\end{align*}
	Since $\mathcal E=\tau^{-1}\mathcal S^{d-1}$,
	\begin{align*}
		\sup_{\eta_1,\eta_2\in\mathcal E}\|\eta_1-\eta_2\| \le \frac{2}{\tau}.
	\end{align*}
	Therefore,
	\begin{align*}
		&\sup_{\eta_1,\eta_2\in\mathcal E} \left| d_{\mathrm{SM}}(\tau\eta_1\|\tau\eta_2\mid\Gamma_N) - B_\kappa(\eta_2\|\eta_1) \right| \\
		& \quad\quad\quad\quad\le 2\sup_{\eta\in\mathcal E}|\widehat\kappa_N(\eta)-\kappa(\eta)| + \frac{2}{\tau}\sup_{\eta\in\mathcal E} \|\nabla\widehat\kappa_N(\eta)-\nabla\kappa(\eta)\|.
	\end{align*}
	By \Cref{eq:uniform-kappa-app} and \Cref{eq:uniform-grad-app}, the right-hand side converges almost surely to zero. Replacing $\eta_a=\lambda_a/\tau$ gives
	\begin{align*}
		\sup_{\lambda_1,\lambda_2\in\mathcal S^{d-1}} \left| d_{\mathrm{SM}}(\lambda_1\|\lambda_2\mid\Gamma_N) - B_\kappa\left(\frac{\lambda_2}{\tau}\Big\|\frac{\lambda_1}{\tau}\right) \right| \xrightarrow[N\to\infty]{\mathrm{a.s.}}0.
	\end{align*}

	This proves the claim.
\end{proof}

\section{Structured approximations to the population anchor geometry}
\label{app:latent_expfam}

The population anchor divergence is determined by the anchor distribution through its cumulant generating function. To obtain tractable estimators of this geometry, we consider structured parametric approximations to the anchor distribution. These models should be understood as pragmatic structured surrogates for the population anchor geometry, not as literal distributional models of normalized contrastive embeddings. The Gaussian and Gaussian-mixture models place mass off the unit sphere, while the latent low-rank models additionally idealize the semantically relevant variation as concentrated in a lower-dimensional subspace. We use these models because they yield tractable geometries with useful statistical tradeoffs in fidelity, robustness, and sample efficiency.

\subsection{General framework}
\label{sec:latent-expfam}

We begin with a latent semantic variable $z \in \mathcal Z \subseteq \mathbb R^k$, representing text-side semantic information, and a corresponding latent natural parameter $\eta \in \mathcal H \subseteq \mathbb R^k$, representing the image-side semantic parameter. Given a base measure $q_Z$ on $\mathcal Z$, we let $\eta$ parameterize $z$ through the exponential family
\begin{align}
	p(dz\mid \eta) = \exp\!\left(\eta^\top z-\kappa_Z(\eta)\right)q_Z(dz),
\label{eq:latent_expfam}
\end{align}
whose cumulant generating function is
\begin{align}
	\kappa_Z(\eta) := \log \int \exp(\eta^\top z)\,q_Z(dz).
\label{eq:latent_cgf}
\end{align}
The associated latent semantic geometry is the Bregman divergence
\begin{align}
	d_Z(\eta_1\|\eta_2) := B_{\kappa_Z}(\eta_2\|\eta_1).
\label{eq:latent_bregman}
\end{align}

To relate this latent construction to the anchor divergence used in the main text, let
\begin{align}
	\gamma = T(z),
\label{eq:latent_to_ambient}
\end{align}
where $T$ maps latent text-side semantics into ambient text embeddings. $\lambda$ is then defined as the reparameterization of $\eta$ so that $\lambda/\tau$ is the natural parameter for the exponential family with sufficient statistic $\gamma$, aligning with the CLIP scoring rule. We consider two cases:

\paragraph{Identity case.}
When $T=\mathrm{Id}$, we have $z=\gamma$, so 
\begin{align}
	p(d\gamma\mid \eta) = \exp\!\left(\eta^\top \gamma-\kappa_Z(\eta)\right)q_Z(d\gamma),
\label{eq:identity_readout}
\end{align}
and hence $\frac{\lambda}{\tau} = \eta$. In this case, \Cref{eq:latent_cgf} reduces to the cumulant generating function of the anchor distribution, and \Cref{eq:latent_bregman} recovers the population anchor divergence from the main text.

\paragraph{Low-rank affine case.}
For the latent low-rank models, we take
\begin{align}
	T(z)=\mu+Uz,
\label{eq:affine_latent_map}
\end{align}
where $U\in\mathbb R^{d\times k}$ has orthonormal columns. Thus, latent semantics are realized in the affine subspace $\mu+\operatorname{Im}(U)\subset\mathbb R^d$. Since $T$ is injective onto this subspace, the latent exponential family induces a pushed-forward family on ambient anchors. Writing $q=T_\# q_Z$ for the pushforward base measure, we obtain
\begin{align}
	p(d\gamma\mid \eta) = \exp\!\left(\eta^\top U^\top(\gamma-\mu)-\kappa_Z(\eta)\right)q(d\gamma).
\label{eq:gamma_given_eta}
\end{align}
To align this family with the CLIP score, we parameterize the image side by
\begin{align}
	\frac{\lambda}{\tau}=U\eta.
\label{eq:lambda_from_eta}
\end{align}
So the pushed-forward family can be written in the score-aligned form
\begin{align}
	p(d\gamma\mid \lambda) = \exp\!\left(\frac{1}{\tau}\lambda^\top \gamma - \Bigl[\kappa_Z\!\left(\frac{1}{\tau}U^\top\lambda\right) + \frac{1}{\tau}\lambda^\top\mu\Bigr]\right)q(d\gamma).
\label{eq:gamma_given_lambda}
\end{align}

The relation $\lambda/\tau=U\eta$ defines an exact inverse only on the semantic subspace $\operatorname{Im}(U)$. Because the low-rank models are idealized approximations, general ambient image embeddings need not lie exactly in this subspace, so we use a model-specific extended inverse or projection map $\eta(\lambda)$. The induced image--image semantic divergence is then defined by pulling back the latent geometry:
\begin{align}
	d(\lambda_1\|\lambda_2) := d_Z\bigl(\eta(\lambda_1)\|\eta(\lambda_2)\bigr) = B_{\kappa_Z}\bigl(\eta(\lambda_2)\|\eta(\lambda_1)\bigr).
\label{eq:pullback_geometry}
\end{align}
Thus the semantic geometry is computed in latent space, while ambient text and image embeddings provide its realization and parameterization.

\subsection{Gaussian anchor model}
\label{sec:app_gaussian_model}

We first recover the ambient-space Gaussian model as the identity case of the general framework. Take
\begin{align}
	z \in \mathbb R^d,\qquad q_Z=\mathcal N(\mu,\Sigma),\qquad T=\mathrm{Id},
\label{eq:gaussian_relations_app}
\end{align}
so that $z=\gamma$. By the identity-case discussion in \Cref{sec:latent-expfam}, the image-side natural parameter is
\begin{align}
	\eta=\frac{\lambda}{\tau}.
\label{eq:gaussian_readout_lambda}
\end{align}
Thus the Gaussian anchor model is simply the case in which the latent semantic space coincides with the ambient anchor space.

\begin{restatable}[Gaussian anchor model]{proposition}{gaussian}
\label{prop:gaussian_metric}
For the Gaussian anchor model,
\begin{equation}
	d_{\mathrm{G}}(\lambda_1\|\lambda_2) = \frac{1}{2\tau^2}(\lambda_1-\lambda_2)^\top \Sigma(\lambda_1-\lambda_2).
\label{eq:gaussian_metric_main}
\end{equation}
\end{restatable}

\begin{proof}
For $q_Z=\mathcal N(\mu,\Sigma)$, the cumulant generating function is
\begin{align}
	\kappa_Z(\eta) &= \log \int \exp(\eta^\top z)\,\mathcal N(dz;\mu,\Sigma) \\
	&= \mu^\top \eta + \frac{1}{2}\eta^\top\Sigma \eta.
\label{eq:gaussian_kappa_app}
\end{align}
Hence
\begin{align}
	\nabla\kappa_Z(\eta)=\mu+\Sigma\eta.
\end{align}
Substituting into \Cref{eq:latent_bregman} gives
\begin{align}
	d_Z(\eta_1\|\eta_2) &= B_{\kappa_Z}(\eta_2\|\eta_1) \\
	&= \frac{1}{2}(\eta_2-\eta_1)^\top \Sigma (\eta_2-\eta_1).
\label{eq:gaussian_bregman_app}
\end{align}
Using $\eta(\lambda)=\lambda/\tau$ in \Cref{eq:pullback_geometry},
\begin{align}
	d_{\mathrm G}(\lambda_1\|\lambda_2) &= d_Z\!\left(\frac{\lambda_1}{\tau}\Big\|\frac{\lambda_2}{\tau}\right) \\
	&= \frac{1}{2\tau^2}(\lambda_1-\lambda_2)^\top\Sigma(\lambda_1-\lambda_2),
\end{align}
which proves \Cref{eq:gaussian_metric_main}.
\end{proof}

\subsection{Latent Gaussian model}
\label{sec:app_latent_gaussian_model}

We now impose Gaussian structure only on a lower-dimensional latent semantic variable. Take
\begin{align}
	z \in \mathbb R^k,\qquad q_Z=\mathcal N(0,I_k),\qquad T(z)=\mu+Uz,
\label{eq:latent_gaussian_base_app}
\end{align}
where $U\in\mathbb R^{d\times k}$ has orthonormal columns. By the low-rank affine construction in \Cref{sec:latent-expfam}, the image-side parameter satisfies
\begin{align}
	\frac{\lambda}{\tau}=U\eta,
\label{eq:latent_gaussian_lambda_from_eta}
\end{align}
and for a general ambient image embedding we use the canonical projection
\begin{align}
	\eta(\lambda)=\frac{1}{\tau}U^\top\lambda.
\label{eq:latent_gaussian_eta_of_lambda_app}
\end{align}
Thus the image--image divergence is obtained by pulling back the latent Gaussian geometry to the semantic subspace.

\begin{restatable}[Latent Gaussian model]{proposition}{latentgaussian}
\label{prop:latent_gaussian_metric}
For the latent Gaussian model with semantic subspace $U$,
\begin{equation}
	d_{\mathrm{LG}}(\lambda_1\|\lambda_2) = \frac{1}{2\tau^2}(\lambda_1-\lambda_2)^\top UU^\top(\lambda_1-\lambda_2).
\label{eq:latent_gaussian_metric_main}
\end{equation}
\end{restatable}

\begin{proof}
For the standard Gaussian base measure on $\mathbb R^k$,
\begin{align}
	\kappa_Z(\eta) &= \log \int \exp(\eta^\top z)\,\mathcal N(dz;0,I_k) \\
	&= \frac{1}{2}\eta^\top\eta.
\label{eq:latent_gaussian_kappa_app}
\end{align}
Hence
\begin{align}
	\nabla\kappa_Z(\eta)=\eta,
\end{align}
and therefore
\begin{align}
	d_Z(\eta_1\|\eta_2) = B_{\kappa_Z}(\eta_2\|\eta_1) = \frac{1}{2}\|\eta_2-\eta_1\|_2^2.
\label{eq:latent_gaussian_bregman_app}
\end{align}
Using \Cref{eq:pullback_geometry} together with \Cref{eq:latent_gaussian_eta_of_lambda_app},
\begin{align}
	d_{\mathrm{LG}}(\lambda_1\|\lambda_2) &= \frac{1}{2}\left\|\frac{1}{\tau}U^\top\lambda_2-\frac{1}{\tau}U^\top\lambda_1\right\|_2^2 \\
	&= \frac{1}{2\tau^2}\|U^\top(\lambda_2-\lambda_1)\|_2^2 \\
	&= \frac{1}{2\tau^2}(\lambda_1-\lambda_2)^\top UU^\top(\lambda_1-\lambda_2),
\end{align}
which proves \Cref{eq:latent_gaussian_metric_main}.
\end{proof}

\subsection{Latent spherical model}
\label{sec:app_latent_spherical_model}

The latent spherical model keeps only directional semantic variation within the same low-dimensional semantic subspace. We take
\begin{align}
	z \in \mathbb S^{k-1}, \qquad \eta \in \frac{1}{\tau}\mathbb S^{k-1}, \qquad q_Z=\mathrm{Unif}(\mathbb S^{k-1}), \qquad T(z)=\mu+Uz,
\label{eq:latent_spherical_base_app}
\end{align}
where $U\in\mathbb R^{d\times k}$ has orthonormal columns. On the semantic subspace, the low-rank affine construction in \Cref{sec:latent-expfam} gives
\begin{align}
	\frac{\lambda}{\tau}=U\eta.
\label{eq:latent_spherical_lambda_from_eta}
\end{align}
For a general ambient image embedding, we extend this parameterization by retaining the direction of its projection while fixing the latent radius to $1/\tau$:
\begin{align}
	\eta(\lambda):=\frac{1}{\tau}\frac{U^\top\lambda}{\|U^\top\lambda\|_2}, \qquad \text{whenever } U^\top\lambda\neq 0.
\label{eq:latent_spherical_eta_of_lambda_app}
\end{align}
Thus the image--image divergence is obtained by pulling back the latent directional geometry to the semantic subspace.

\begin{restatable}[Latent spherical model]{proposition}{latentspherical}
\label{prop:latent_spherical_metric}
For the latent spherical model with semantic subspace $U$, and assuming $U^\top\lambda_1\neq 0$ and $U^\top\lambda_2\neq 0$,
\begin{equation}
	d_{\mathrm{LS}}(\lambda_1\|\lambda_2)=\frac{1}{\tau}A_k\left(\frac{1}{\tau}\right)\left(1-\frac{(U^\top\lambda_1)^\top(U^\top\lambda_2)}{\|U^\top\lambda_1\|_2\,\|U^\top\lambda_2\|_2}\right),
\label{eq:latent_spherical_metric_main}
\end{equation}
where
\begin{align*}
	A_k(r)=\frac{I_{k/2}(r)}{I_{k/2-1}(r)},
\end{align*}
and $I_\nu$ denotes the modified Bessel function of the first kind.
\end{restatable}

\begin{proof}
For the uniform base measure on $\mathbb S^{k-1}$, the cumulant generating function depends only on the norm:
\begin{align}
	\kappa_Z(\eta)=\psi_k(\|\eta\|_2),
\label{eq:latent_spherical_kappa_app}
\end{align}
where $\psi_k$ is the usual von Mises--Fisher log-normalizer. Its gradient is radial:
\begin{align}
	\nabla\kappa_Z(\eta)=A_k(\|\eta\|_2)\frac{\eta}{\|\eta\|_2}.
\label{eq:latent_spherical_gradkappa_app}
\end{align}
Now let $\eta_1,\eta_2\in\frac{1}{\tau}\mathbb S^{k-1}$ and write $r=1/\tau$. Since $\|\eta_1\|_2=\|\eta_2\|_2=r$,
\begin{align}
	\kappa_Z(\eta_1)=\kappa_Z(\eta_2)=\psi_k(r), \qquad \nabla\kappa_Z(\eta_1)=A_k(r)\frac{\eta_1}{r}.
\label{eq:latent_spherical_equal_radius_app}
\end{align}
Therefore,
\begin{align}
	d_Z(\eta_1\|\eta_2)&=B_{\kappa_Z}(\eta_2\|\eta_1) \\
	&=\kappa_Z(\eta_2)-\kappa_Z(\eta_1)-\nabla\kappa_Z(\eta_1)^\top(\eta_2-\eta_1) \\
	&=-\frac{A_k(r)}{r}\eta_1^\top(\eta_2-\eta_1) \\
	&=rA_k(r)\left(1-\frac{\eta_1^\top\eta_2}{r^2}\right).
\label{eq:latent_spherical_bregman_app}
\end{align}
Using \Cref{eq:pullback_geometry} together with \Cref{eq:latent_spherical_eta_of_lambda_app},
\begin{align}
	d_{\mathrm{LS}}(\lambda_1\|\lambda_2)&=\frac{1}{\tau}A_k\left(\frac{1}{\tau}\right)\left(1-\tau^2\eta(\lambda_1)^\top\eta(\lambda_2)\right) \\
	&=\frac{1}{\tau}A_k\left(\frac{1}{\tau}\right)\left(1-\frac{(U^\top\lambda_1)^\top(U^\top\lambda_2)}{\|U^\top\lambda_1\|_2\,\|U^\top\lambda_2\|_2}\right),
\end{align}
which proves \Cref{eq:latent_spherical_metric_main}.
\end{proof}

\subsection{Gaussian mixture model}
\label{sec:app_gmm_model}

Finally, we obtain a more expressive ambient-space approximation by modeling the anchor distribution as a Gaussian mixture. Take
\begin{align}
	z \in \mathbb R^d, \qquad q_Z = \sum_{m=1}^M \pi_m\,\mathcal N(\mu_m,\Sigma_m), \qquad T=\mathrm{Id},
\label{eq:gmm_relations_app}
\end{align}
where $\pi_m>0$, $\sum_{m=1}^M \pi_m=1$, and $\Sigma_m\succ 0$. Thus, $z=\gamma$. By the identity-case construction in \Cref{sec:latent-expfam}, the image-side natural parameter is
\begin{align}
	\eta=\frac{\lambda}{\tau}.
\label{eq:gmm_readout_lambda}
\end{align}
This yields a multimodal ambient approximation to the population anchor geometry.

\begin{restatable}[Gaussian mixture anchor model]{proposition}{gmm}
\label{prop:gmm_metric}
For the Gaussian mixture anchor model,
\begin{equation}
	d_{\mathrm{GMM}}(\lambda_1\|\lambda_2) = B_{\kappa_{\mathrm{GMM}}}\!\left(\frac{\lambda_2}{\tau}\Big\|\frac{\lambda_1}{\tau}\right),
\label{eq:gmm_metric_main}
\end{equation}
where
\begin{equation}
	\kappa_{\mathrm{GMM}}(\eta) = \log\sum_{m=1}^M \pi_m \exp\!\left(\mu_m^\top \eta+\frac{1}{2}\eta^\top\Sigma_m\eta\right).
\label{eq:gmm_kappa_main}
\end{equation}
\end{restatable}

\begin{proof}
Starting from \Cref{eq:gmm_relations_app},
\begin{align}
	\kappa_Z(\eta) &= \log \int \exp(\eta^\top z)\,q_Z(dz) \\
	&= \log \sum_{m=1}^M \pi_m \int \exp(\eta^\top z)\,\mathcal N(dz;\mu_m,\Sigma_m).
\end{align}
Each inner integral is the Gaussian MGF, so
\begin{align}
	\kappa_Z(\eta) = \log\sum_{m=1}^M \pi_m \exp\!\left(\mu_m^\top \eta+\frac{1}{2}\eta^\top\Sigma_m\eta\right),
\end{align}
which is exactly \Cref{eq:gmm_kappa_main}. Hence,
\begin{align}
	d_Z(\eta_1\|\eta_2) = B_{\kappa_{\mathrm{GMM}}}(\eta_2\|\eta_1).
\end{align}
Using $\eta(\lambda)=\lambda/\tau$ in \Cref{eq:pullback_geometry},
\begin{align}
	d_{\mathrm{GMM}}(\lambda_1\|\lambda_2) &= d_Z\!\left(\frac{\lambda_1}{\tau}\Big\|\frac{\lambda_2}{\tau}\right) \\
	&= B_{\kappa_{\mathrm{GMM}}}\!\left(\frac{\lambda_2}{\tau}\Big\|\frac{\lambda_1}{\tau}\right),
\end{align}
which proves \Cref{eq:gmm_metric_main}.
\end{proof}

\section{Parameter estimation for structured anchor models}
\label{app:param_estimation}

Given anchor embeddings $\{\gamma_i\}_{i=1}^n \subset \mathbb{R}^d$, we estimate the parameters of each structured anchor model directly from the anchor set.

\paragraph{Gaussian anchor model.}
For the Gaussian model, we fit the mean and covariance by maximum likelihood:
\begin{align}
\hat \mu &= \frac{1}{n}\sum_{i=1}^n \gamma_i, \\
\hat \Sigma &= \frac{1}{n}\sum_{i=1}^n (\gamma_i-\hat\mu)(\gamma_i-\hat\mu)^\top.
\end{align}
The fitted semantic divergence is obtained by substituting $\hat\Sigma$ into \Cref{eq:gaussian_metric_main}.

\paragraph{Low-rank models.}
For the latent Gaussian and latent spherical models, we estimate the semantic subspace $U \in \mathbb{R}^{d\times k}$ by the top-$k$ eigenvectors of the anchor covariance matrix
\begin{align}
\hat \Sigma = \frac{1}{n}\sum_{i=1}^n (\gamma_i-\bar\gamma)(\gamma_i-\bar\gamma)^\top,
\qquad
\bar\gamma = \frac{1}{n}\sum_{i=1}^n \gamma_i.
\end{align}
We choose $k$ to explain $95\%$ of the anchor covariance. The fitted divergences are then obtained by substituting $\hat U$ into \Cref{eq:latent_gaussian_metric_main} and \Cref{eq:latent_spherical_metric_main}, respectively.

\paragraph{Gaussian mixture model.}
For the Gaussian mixture model, we fit the mixture weights, means, and covariances,
\begin{align*}
\{(\pi_m,\mu_m,\Sigma_m)\}_{m=1}^M,
\end{align*}
by expectation--maximization (EM) on the anchor embeddings. The fitted parameters are then substituted into \Cref{eq:gmm_kappa_main} and \Cref{eq:gmm_metric_main}.

\section{Model-specific posterior and posterior-like updates for posterior anchor modeling}
\label{app:adaptation_appendix}

This appendix gives model-specific posterior and posterior-like updates for the posterior anchor modeling framework introduced in \Cref{sec:adaptation}. Broad anchors provide prior semantic structure for a parametric anchor model, precise task-aligned anchors refine that structure, and the resulting model summary induces an adapted geometry through  $B_{\kappa_{\hat\theta}}$.

\subsection{General adaptation framework}

Suppose we are given a broad anchor set
\begin{align*}
	\Gamma_{\mathrm{broad}} = \{\gamma^{\mathrm{broad}}_1,\dots,\gamma^{\mathrm{broad}}_m\}
\end{align*}
and a smaller set of precise task-aligned anchors
\begin{align*}
	\Gamma_{\mathrm{target}} = \{\gamma^{\mathrm{target}}_1,\dots,\gamma^{\mathrm{target}}_n\}.
\end{align*}
Let $\{q_\theta:\theta\in\Theta\}$ be a parametric family of anchor distributions together with model-specific prior structure on $\Theta$. A generic adaptation procedure is:

\begin{enumerate}
    \item fit the parametric family on $\Gamma_{\mathrm{broad}}$ to obtain a broad-anchor estimate $\theta_{\mathrm{broad}}$;

    \item choose hyperparameters or regularization structure $\Omega=\Omega(\theta_{\mathrm{broad}})$ so that the update is centered around the broad semantic structure encoded by $\theta_{\mathrm{broad}}$;

    \item combine this prior structure with the precise anchors to obtain a model-specific posterior or posterior-like update over $\theta$;

    \item use a model-specific summary $\hat\theta$ to define the adapted anchor model $q_{\hat\theta}$;

    \item induce the adapted geometry through
    \begin{align*}
    	\kappa_{\hat\theta}(\xi) = \log \int \exp(\xi^\top \gamma)\,q_{\hat\theta}(d\gamma), \qquad d_{\hat\theta}(\lambda_1\|\lambda_2) = B_{\kappa_{\hat\theta}}\left( \frac{\lambda_2}{\tau} \Big\| \frac{\lambda_1}{\tau} \right).
    \end{align*}
\end{enumerate}

We now specialize this framework to the model families used in the paper.

\subsection{Gaussian anchor model with Normal--Inverse-Wishart update}
\label{app:adapt_gaussian}

For the Gaussian anchor model,
\begin{align*}
	q_\theta(\gamma)=\mathcal N(\gamma;\mu,\Sigma), \qquad \theta=(\mu,\Sigma), \qquad \Sigma\succ 0.
\end{align*}
As shown in \Cref{sec:app_gaussian_model}, the induced geometry is
\begin{align*}
	d_{\mathrm G}(\lambda_1\|\lambda_2)
	=
	\frac{1}{2\tau^2}
	(\lambda_1-\lambda_2)^\top
	\Sigma
	(\lambda_1-\lambda_2),
\end{align*}
so adaptation corresponds to updating the covariance structure of the anchor distribution.

We place a Normal--Inverse-Wishart prior on $(\mu,\Sigma)$:
\begin{align*}
	\Sigma &\sim \mathcal{IW}(\Psi_0,\nu_0), \\
	\mu\mid \Sigma &\sim \mathcal N\!\left(\mu_0,\frac{1}{\kappa_0}\Sigma\right).
\end{align*}

The adaptation steps are:

\begin{enumerate}
    \item \textbf{Fit the broad Gaussian.} Fit a Gaussian model to the broad anchors:
    \begin{align*}
    	\bar\gamma_{\mathrm{broad}}
    	&=
    	\frac{1}{m}\sum_{i=1}^m \gamma_i^{\mathrm{broad}}, \\
    	S_{\mathrm{broad}}
    	&=
    	\sum_{i=1}^m
    	(\gamma_i^{\mathrm{broad}}-\bar\gamma_{\mathrm{broad}})
    	(\gamma_i^{\mathrm{broad}}-\bar\gamma_{\mathrm{broad}})^\top.
    \end{align*}
    This gives the broad-anchor estimate $(\mu_{\mathrm{broad}},\Sigma_{\mathrm{broad}})$.

    \item \textbf{Initialize the prior.} Use the broad-anchor estimate in an empirical-Bayes step:
    \begin{align*}
    	\mu_0 &= \bar\gamma_{\mathrm{broad}}, \\
    	\Psi_0 &= (\nu_0-d-1)\Sigma_{\mathrm{broad}},
    \end{align*}
    where $\kappa_0$ and $\nu_0$ control prior strength.

    \item \textbf{Update with precise anchors.} Given precise anchors $\Gamma_{\mathrm{target}}$, compute
    \begin{align*}
    	\bar\gamma_{\mathrm{target}}
    	&=
    	\frac{1}{n}\sum_{i=1}^n \gamma_i^{\mathrm{target}}, \\
    	S_{\mathrm{target}}
    	&=
    	\sum_{i=1}^n
    	(\gamma_i^{\mathrm{target}}-\bar\gamma_{\mathrm{target}})
    	(\gamma_i^{\mathrm{target}}-\bar\gamma_{\mathrm{target}})^\top.
    \end{align*}
    The posterior is again NIW with parameters
    \begin{align*}
    	\kappa_n &= \kappa_0+n, \\
    	\nu_n &= \nu_0+n, \\
    	\mu_n &= \frac{\kappa_0\mu_0+n\bar\gamma_{\mathrm{target}}}{\kappa_0+n}, \\
    	\Psi_n &= \Psi_0 + S_{\mathrm{target}} + \frac{\kappa_0 n}{\kappa_0+n} (\bar\gamma_{\mathrm{target}}-\mu_0)(\bar\gamma_{\mathrm{target}}-\mu_0)^\top.
    \end{align*}

    \item \textbf{Form the adapted model.} Use a posterior summary $(\hat\mu,\hat\Sigma)$, such as the posterior mean or MAP estimate, to define the adapted Gaussian anchor model
    \begin{align*}
    	q_{\hat\theta}(\gamma)=\mathcal N(\gamma;\hat\mu,\hat\Sigma).
    \end{align*}

    \item \textbf{Induce the adapted geometry.} The adapted geometry is
    \begin{align*}
    	d_{\mathrm G}^{\mathrm{adapt}}(\lambda_1\|\lambda_2)
    	=
    	\frac{1}{2\tau^2}
    	(\lambda_1-\lambda_2)^\top
    	\hat\Sigma
    	(\lambda_1-\lambda_2).
    \end{align*}
\end{enumerate}

\subsection{Latent Gaussian model with matrix Bingham update}
\label{app:adapt_latent_gaussian}

For the latent Gaussian model, the anchor distribution is determined by a low-dimensional semantic subspace. Concretely, we use the ambient surrogate
\begin{align*}
	q_\theta(\gamma)=\mathcal N(\gamma;\mu,UU^\top+\varepsilon I), \qquad \theta=(\mu,U),
\end{align*}
where $U\in\mathbb R^{d\times k}$ has orthonormal columns, $\mu\in\mathbb R^d$, and $\varepsilon>0$ is a small isotropic regularization used in implementation. As shown in \Cref{sec:app_latent_gaussian_model}:
\begin{align*}
	d_{\mathrm{LG}}(\lambda_1\|\lambda_2)
	=
	\frac{1}{2\tau^2}
	(\lambda_1-\lambda_2)^\top
	UU^\top
	(\lambda_1-\lambda_2).
\end{align*}
Thus, adaptation amounts to updating the subspace $U$.

We place a matrix Bingham prior on $U$ over the Stiefel manifold
\begin{align*}
	\mathrm{St}(d,k)=\{U\in\mathbb R^{d\times k}:U^\top U=I_k\},
\end{align*}
with density
\begin{align*}
	p(U\mid A)\propto \exp\!\bigl(\operatorname{tr}(U^\top A U)\bigr).
\end{align*}
Under the Gaussian covariance model above, this prior is conjugate for the subspace parameter $U$.

The adaptation steps are:

\begin{enumerate}
    \item \textbf{Fit the broad semantic subspace.} Compute the empirical covariance of the broad anchors and let $U_{\mathrm{broad}}$ be its top-$k$ eigenvectors:
    \begin{align*}
    	\widehat{\operatorname{Cov}}(\Gamma_{\mathrm{broad}})
    	=
    	\frac{1}{m}\sum_{i=1}^m
    	(\gamma_i^{\mathrm{broad}}-\bar\gamma_{\mathrm{broad}})
    	(\gamma_i^{\mathrm{broad}}-\bar\gamma_{\mathrm{broad}})^\top.
    \end{align*}

    \item \textbf{Initialize the Bingham prior.} Use the broad subspace to define
    \begin{align*}
    	A=\alpha\,U_{\mathrm{broad}}U_{\mathrm{broad}}^\top,
    \end{align*}
    where $\alpha>0$ controls prior strength.

    \item \textbf{Update with precise anchors.} Let
    \begin{align*}
    	\widehat{\operatorname{Cov}}(\Gamma_{\mathrm{target}})
    	=
    	\frac{1}{n}\sum_{i=1}^n
    	(\gamma_i^{\mathrm{target}}-\bar\gamma_{\mathrm{target}})
    	(\gamma_i^{\mathrm{target}}-\bar\gamma_{\mathrm{target}})^\top.
    \end{align*}
    Under the Gaussian covariance model, the posterior over $U$ remains in matrix-Bingham form, with updated parameter matrix
    \begin{align*}
    	A_{\mathrm{post}}
    	=
    	A+\beta\,\widehat{\operatorname{Cov}}(\Gamma_{\mathrm{target}}),
    \end{align*}
    where
    \begin{align*}
    	\beta
    	=
    	\frac{n}{2}\left(\frac{1}{\varepsilon}-\frac{1}{1+\varepsilon}\right).
    \end{align*}

    \item \textbf{Form the adapted model via MAP.} We use the MAP estimator under the Bingham posterior:
    \begin{align*}
    	U_{\mathrm{post}}
    	=
    	\arg\max_{U\in \mathrm{St}(d,k)}
    	\operatorname{tr}(U^\top A_{\mathrm{post}}U),
    \end{align*}
    which is given by the top-$k$ eigenvectors of $A_{\mathrm{post}}$.

    \item \textbf{Induce the adapted geometry.} The adapted geometry is
    \begin{align*}
    	d_{\mathrm{LG}}^{\mathrm{adapt}}(\lambda_1\|\lambda_2)
    	=
    	\frac{1}{2\tau^2}
    	(\lambda_1-\lambda_2)^\top
    	U_{\mathrm{post}}U_{\mathrm{post}}^\top
    	(\lambda_1-\lambda_2).
    \end{align*}
\end{enumerate}

\subsection{Latent spherical model with posterior-like Bingham update}
\label{app:adapt_latent_spherical}

For the latent spherical model, adaptation again reduces to updating the semantic subspace $U$. We use the same matrix-Bingham form as a tractable posterior-like update, driven by the broad-anchor subspace and the precise-anchor covariance. As in the latent Gaussian case, we initialize
\begin{align*}
	A=\alpha\,U_{\mathrm{broad}}U_{\mathrm{broad}}^\top
\end{align*}
from the broad-anchor subspace $U_{\mathrm{broad}}$, then combine this with the precise-anchor covariance to form
\begin{align*}
	A_{\mathrm{post}} = A + \beta\,\widehat{\operatorname{Cov}}(\Gamma_{\mathrm{target}}),
\end{align*}
where $\beta>0$ controls the influence of the precise anchors. The adapted subspace is the MAP estimate
\begin{align*}
	U_{\mathrm{post}}
	=
	\arg\max_{U\in\mathrm{St}(d,k)}
	\operatorname{tr}(U^\top A_{\mathrm{post}}U),
\end{align*}
given by the top-$k$ eigenvectors of $A_{\mathrm{post}}$. The adapted geometry is obtained by substituting $U_{\mathrm{post}}$ into the latent spherical divergence.

\subsection{Gaussian mixture model with Dirichlet--NIW update}
\label{app:adapt_gmm}

For the Gaussian mixture model, we parameterize the anchor distribution as
\begin{align*}
	q_\theta(d\gamma)
	=
	\sum_{m=1}^M \pi_m\,\mathcal N(d\gamma;\mu_m,\Sigma_m),
	\qquad
	\theta=\{(\pi_m,\mu_m,\Sigma_m)\}_{m=1}^M,
\end{align*}
where $\pi_m>0$, $\sum_{m=1}^M \pi_m=1$, and $\Sigma_m\succ 0$. The induced geometry is the Bregman divergence generated by
\begin{align*}
	\kappa_{\mathrm{GMM}}(\xi)
	=
	\log\sum_{m=1}^M
	\pi_m
	\exp\!\left(
	\mu_m^\top \xi + \frac{1}{2}\xi^\top \Sigma_m \xi
	\right).
\end{align*}
Unlike the single-Gaussian case, this geometry is generally nonquadratic and can represent multiple semantic modes.

We place a symmetric Dirichlet prior on the mixture weights and independent Normal--Inverse-Wishart priors on the component parameters:
\begin{align*}
	(\pi_1,\dots,\pi_M) &\sim \operatorname{Dir}(\alpha,\dots,\alpha), \\
	\Sigma_m &\sim \mathcal{IW}(\Psi_{0,m},\nu_{0,m}), \\
	\mu_m\mid \Sigma_m &\sim \mathcal N\!\left(\mu_{0,m},\frac{1}{\kappa_{0,m}}\Sigma_m\right),
	\qquad m=1,\dots,M.
\end{align*}

The adaptation steps are:

\begin{enumerate}
    \item \textbf{Fit the broad mixture.} Fit an $M$-component GMM to the broad anchors, producing
    \begin{align*}
    	\left\{
    	\pi_m^{\mathrm{broad}},\mu_m^{\mathrm{broad}},\Sigma_m^{\mathrm{broad}}
    	\right\}_{m=1}^M.
    \end{align*}

    \item \textbf{Initialize the prior.} Use the broad-anchor fit to set the component-specific hyperparameters
    \begin{align*}
    	\mu_{0,m} &= \mu_m^{\mathrm{broad}}, \\
    	\Psi_{0,m} &= (\nu_{0,m}-d-1)\Sigma_m^{\mathrm{broad}},
    \end{align*}
    while resetting the mixture weights to a symmetric prior
    \begin{align*}
    	(\pi_1,\dots,\pi_M)\sim \operatorname{Dir}(\alpha,\dots,\alpha).
    \end{align*}
    Thus, the broad anchors determine the semantic modes and local covariance scales, but not their prior prevalence.

    \item \textbf{Update with precise anchors using soft assignments.} Given precise anchors $\Gamma_{\mathrm{target}}$, compute responsibilities
    \begin{align*}
    	r_{im}=p(m\mid \gamma_i^{\mathrm{target}})
    \end{align*}
    under the current mixture model. Soft assignments allow precise anchors to refine multiple broad semantic modes simultaneously. Define weighted sufficient statistics
    \begin{align*}
    	N_m &= \sum_{i=1}^n r_{im}, \\
    	\bar\gamma_m &= \frac{1}{N_m}\sum_{i=1}^n r_{im}\gamma_i^{\mathrm{target}}, \\
    	S_m &= \sum_{i=1}^n r_{im}(\gamma_i^{\mathrm{target}}-\bar\gamma_m)(\gamma_i^{\mathrm{target}}-\bar\gamma_m)^\top.
    \end{align*}
    The updated hyperparameters are
    \begin{align*}
    	\alpha_m' &= \alpha + N_m, \\
    	\kappa_m' &= \kappa_{0,m}+N_m, \\
    	\nu_m' &= \nu_{0,m}+N_m, \\
    	\mu_m' &= \frac{\kappa_{0,m}\mu_{0,m}+N_m\bar\gamma_m}{\kappa_{0,m}+N_m}, \\
    	\Psi_m' &= \Psi_{0,m} + S_m + \frac{\kappa_{0,m}N_m}{\kappa_{0,m}+N_m}(\bar\gamma_m-\mu_{0,m})(\bar\gamma_m-\mu_{0,m})^\top.
    \end{align*}

    \item \textbf{Form the adapted model.} Use posterior means to define the adapted mixture parameters
    \begin{align*}
    	\hat\pi_m &= \frac{\alpha_m'}{\sum_{\ell=1}^M \alpha_\ell'}, \\
    	\hat\mu_m &= \mu_m', \\
    	\hat\Sigma_m &= \frac{\Psi_m'}{\nu_m'-d-1},
    \end{align*}
    provided $\nu_m'>d+1$. This gives the adapted anchor model
    \begin{align*}
    	q_{\hat\theta}(d\gamma)
    	=
    	\sum_{m=1}^M \hat\pi_m\,\mathcal N(d\gamma;\hat\mu_m,\hat\Sigma_m).
    \end{align*}

    \item \textbf{Induce the adapted geometry.} The adapted geometry is
    \begin{align*}
    	d_{\mathrm{GMM}}^{\mathrm{adapt}}(\lambda_1\|\lambda_2)
    	=
    	B_{\kappa_{\hat\theta}}\!\left(
    	\frac{\lambda_2}{\tau}\Big\|\frac{\lambda_1}{\tau}
    	\right),
    \end{align*}
    where
    \begin{align*}
    	\kappa_{\hat\theta}(\xi)
    	=
    	\log\sum_{m=1}^M
    	\hat\pi_m
    	\exp\!\left(
    	\hat\mu_m^\top \xi + \frac{1}{2}\xi^\top \hat\Sigma_m \xi
    	\right).
    \end{align*}
\end{enumerate}

\subsection{Discussion}

These model-specific updates instantiate the same principle: posterior anchor modeling updates a parametric anchor model using broad anchors as prior semantic structure and precise anchors as task-specific evidence, which in turn updates the induced Bregman geometry in image space. The models differ in the semantic structure they can transfer: a single global covariance in the Gaussian case, a shared low-dimensional semantic subspace in the low-rank case, and multiple reweighted semantic modes in the mixture case. This is precisely why richer models such as the GMM can support forms of semantic transfer that would not arise from simple projection heuristics.

\section{Implementation details}
\label{app:implementation}

\subsection{Datasets}
\label{app:datasets}

We evaluate on ImageNet, ImageNet-R, CUB-200, Sketchy, and CheXpert.

\paragraph{ImageNet.}
We use the standard training split as the retrieval gallery/keys and the validation split as queries. Relevance is defined by class agreement. ImageNet class-name anchors are formed from the class labels using the template ``a photo of a \{x\}''.

\paragraph{ImageNet-R.}
Queries and keys are both drawn from ImageNet-R, with each query removed from its own key set. Object relevance is defined by ImageNet class agreement, while style labels for the style-control experiments are parsed from the ImageNet-R filenames. Object-name anchors use the template ``a photo of a \{x\}'', and style anchors use ``a \{s\} rendition of an object''. For style retrieval we discard images that have ``misc'' style labels, since these are not well-defined styles and do not have corresponding anchors.

\paragraph{CUB-200.}
We use the standard train/test split, with training images as the keys and test images as queries. Relevance is defined by bird-class agreement. CUB anchors are formed from the class names using the same class-name template.

\paragraph{Sketchy.}
We use the Sketchy Extended Split 2 protocol for zero-shot sketch-based image retrieval, following \citet{yelamarthi2018zero}. The test set consists of 21 classes chosen to be absent from ImageNet-1K. We use sketches from these unseen classes as queries and photos from the same classes as the retrieval gallery; no training-class images are used. Relevance is defined by class agreement, and anchors are formed from the class names using the template ``a photo of a \{x\}''.

\paragraph{CheXpert.}
We evaluate multi-label retrieval on CheXpert using pathology labels as relevance targets. The official training split is used as keys and the validation split as queries. We discard the training images with uncertain labels. For each target pathology, we construct report-style text anchors describing that pathology. The exact LLM prompt is given in \Cref{app:chexpert_anchors}.

\paragraph{Synthetic collection for Figure~1.}
\Cref{fig:intro_semantic_neighborhoods} uses a synthetic image collection in which dog breed, background, and clothing are systematically varied. Images were generated using OpenAI's \texttt{gpt-image-2} model. The displayed neighbors are obtained directly by retrieval under cosine similarity and the corresponding anchor divergences.

\subsection{Models}
\label{app:models}

All natural-image experiments use MetaCLIP-2. All image and text embeddings are normalized to unit $\ell_2$ norm, and we use the model-provided logit scale, $\exp(\texttt{clip\_model.logit\_scale})$, as the inverse temperature.

For CheXpert, we use CheXzero and report results averaged over the 10 released model weights.

\subsection{Image-text inversion baseline}
\label{app:iti_baseline}

As a retrieval baseline, we use image-text inversion (ITI), which converts a query image into a text-side proxy by selecting the anchor with highest image--text compatibility. Given an anchor set $\Gamma=\{\gamma_1,\dots,\gamma_N\}$ and a query image embedding $\lambda_q$, we first choose
\begin{align}
i^\star(\lambda_q)
=
\arg\max_{1\le i\le N}\;
\lambda_q^\top \gamma_i.
\end{align}
We then take the corresponding anchor
\begin{align}
\tilde{\gamma}(\lambda_q)=\gamma_{i^\star(\lambda_q)}
\end{align}
as the inverted text representation of the query. Each gallery image embedding $\lambda_k$ is then scored by cosine similarity to this inferred text embedding:
\begin{align}
s_{\mathrm{ITI}}(\lambda_q,\lambda_k)
=
\tilde{\gamma}(\lambda_q)^\top \lambda_k.
\end{align}
Keys are ranked in decreasing order of $s_{\mathrm{ITI}}(\lambda_q,\lambda_k)$.

Thus, ITI uses the anchor set only to select a single best-matching text proxy for the query, after which retrieval is performed by standard image-text cosine similarity.

\subsection{Evaluation metric}
\label{app:evaluation_metric}

We evaluate retrieval performance with the area under the precision--recall curve (AUPRC). For a given query, let
\begin{align*}
\mathcal G = \{x_1,\dots,x_N\}
\end{align*}
denote the key set, let $s_i$ be the retrieval score assigned to key item $x_i$, and let
\begin{align*}
y_i \in \{0,1\}
\end{align*}
indicate whether $x_i$ is semantically relevant to the query. Sorting the keys by score induces a ranked list, and for any threshold $t$ on the score we define
\begin{align}
\operatorname{Precision}(t)
&=
\frac{\sum_{i=1}^N \mathbf 1\{s_i \ge t\}y_i}
{\sum_{i=1}^N \mathbf 1\{s_i \ge t\}}, \\
\operatorname{Recall}(t)
&=
\frac{\sum_{i=1}^N \mathbf 1\{s_i \ge t\}y_i}
{\sum_{i=1}^N y_i}.
\end{align}
The precision--recall curve is the set of points
\begin{align*}
\bigl(\operatorname{Recall}(t),\operatorname{Precision}(t)\bigr)
\end{align*}
obtained as the threshold varies over the ranked list, and AUPRC is the area under this curve:
\begin{align}
\operatorname{AUPRC}
=
\int_0^1 \operatorname{Precision}(r)\,dr,
\end{align}
with the discrete implementation given by the corresponding finite-sum approximation over ranked thresholds.

\subsection{Parametric models and adaptation}
\label{app:parametric_impl}

\paragraph{Gaussian model.}
We fit the Gaussian anchor model by maximum likelihood on the anchor embeddings.

\paragraph{Low-rank models.}
For the latent Gaussian and latent spherical models, we estimate the semantic subspace $U$ by the top-$k$ eigenvectors of the centered anchor covariance matrix. Unless otherwise specified, we choose $k$ to explain $95\%$ of the anchor covariance. In the semantic adaptation experiments, where a fixed latent dimension is required, we use $k=128$. For the ambient surrogate used in latent Gaussian adaptation, we set the isotropic regularization to $\varepsilon=10^{-8}$.

\paragraph{Gaussian mixture model.}
We fit the Gaussian mixture model by expectation--maximization, with the number of mixture components specified in the corresponding plots. EM is initialized with $k$-means, each covariance matrix is regularized by adding $10^{-5}I$, and optimization is terminated when the improvement in log-likelihood falls below $10^{-3}$ or when 200 iterations are reached.

\paragraph{Adaptation.}
Gaussian adaptation uses the Normal--Inverse-Wishart update in \Cref{app:adapt_gaussian}. Latent Gaussian adaptation uses the matrix-Bingham update in \Cref{app:adapt_latent_gaussian}. Latent spherical adaptation uses the posterior-like Bingham update in \Cref{app:adapt_latent_spherical}. Gaussian mixture adaptation uses the Dirichlet--NIW update in \Cref{app:adapt_gmm}.

\paragraph{Adaptation hyperparameters.}
For Gaussian adaptation, we use $\kappa_0=5$ and $\nu_0=d + 20$. The low-rank Bingham updates have prior strength $\alpha=50$ and $\varepsilon=10^{-8}$. For GMM adaptation, we use Dirichlet concentration $\alpha=1$, $\kappa_{0,m}=5$, and $\nu_{0,m}=d + 20$, shared across mixture components unless otherwise specified.

\subsection{CheXpert anchor generation}
\label{app:chexpert_anchors}

For CheXpert, we generate pathology-specific report-style anchors with ChatGPT 5.4 Thinking. The prompt asks for a Python dictionary whose keys are the CheXpert conditions and whose values are lists of possible doctor report notes for the corresponding condition. Each list contains both positive and negative entries and is sized to contain roughly 50 strings. The resulting pathology-specific anchor lists are then embedded with CheXzero and used as the anchor sets defining the corresponding retrieval geometries.

\paragraph{Prompt.}
\begin{quote}
For the conditions in LABELS = [``Enlarged Cardiomediastinum'', ``Cardiomegaly'', ``Lung Opacity'', ``Lung Lesion'', ``Edema'', ``Consolidation'', ``Pneumonia'', ``Atelectasis'', ``Pneumothorax'', ``Pleural Effusion'', ``Pleural Other'', ``Fracture'', ``Support Devices''] Give a python dictionary where the keys are the different conditions and the values are a list of strings of different possible doctor report notes for that condition. Have positive and negative entries. Have each list be of size around 50
\end{quote}

\paragraph{Usage in retrieval.}
For each target pathology, we extract the corresponding list from the generated dictionary, embed those report-style strings with CheXzero, and use the resulting text embeddings as the pathology-specific anchor set for retrieval.

\section{Additional Experiments}
\label{app:additional_experiments}

\subsection{Parametric anchor models on CheXpert}
\label{app:chexpert_parametric_models}

The main text reports CheXpert results for the pathology-specialized softmax anchor divergence. Here we include the corresponding parametric anchor models on the same retrieval task. \Cref{tab:chexpert_all_diffs_cosine} shows that Gaussian anchor divergence performs comparably to softmax on average, while the latent models also improve over cosine, indicating that parametric anchor geometries can preserve clinically relevant coordinates beyond the natural-image setting.

\begin{table}[t]

\centering
\small
\setlength{\tabcolsep}{4pt}
\renewcommand{\arraystretch}{0.98}

\caption{
\textbf{Parametric CheXpert anchor divergences.}
Entries are pathology-wise mean AUPRC differences relative to cosine similarity, with 95\% confidence interval.
Results average over independently generated anchor sets and CheXzero checkpoints, with intervals accounting for bootstrap and anchor-generation variability.
Softmax and Gaussian anchor divergences give the largest average gains, while latent models also improve over cosine on average.
}
\label{tab:chexpert_all_diffs_cosine}

\vspace{0.5em}

\resizebox{\textwidth}{!}{
\begin{tabular}{lccccc}
\toprule
\textbf{Pathology}
&
\textbf{ITI $-$ Cosine}
&
\textbf{Softmax $-$ Cosine}
&
\textbf{Gaussian $-$ Cosine}
&
\textbf{Latent G. $-$ Cosine}
&
\textbf{Latent S. $-$ Cosine}
\\
\midrule

Enlarged Cardiomediastinum
& \neu{$0.0042 \pm 0.0009$}
& \neu{$-0.0002 \pm 0.0003$}
& \neu{$-0.0002 \pm 0.0003$}
& \neu{$0.0000 \pm 0.0002$}
& \neu{$-0.0004 \pm 0.0003$}
\\

Cardiomegaly
& \pos{$0.0211 \pm 0.0083$}
& \pos{$0.0279 \pm 0.0034$}
& \pos{$0.0327 \pm 0.0030$}
& \pos{$0.0152 \pm 0.0015$}
& \pos{$0.0194 \pm 0.0020$}
\\

Lung Opacity
& \neu{$-0.0047 \pm 0.0050$}
& \pos{$0.0090 \pm 0.0032$}
& \pos{$0.0092 \pm 0.0031$}
& \neu{$0.0037 \pm 0.0016$}
& \pos{$0.0083 \pm 0.0019$}
\\

Lung Lesion
& \neu{$-0.0050 \pm 0.0011$}
& \neu{$0.0001 \pm 0.0005$}
& \neu{$-0.0000 \pm 0.0004$}
& \neu{$0.0003 \pm 0.0003$}
& \neu{$0.0013 \pm 0.0004$}
\\

Edema
& \neu{$0.0030 \pm 0.0064$}
& \pos{$0.0316 \pm 0.0044$}
& \pos{$0.0309 \pm 0.0045$}
& \pos{$0.0127 \pm 0.0024$}
& \pos{$0.0194 \pm 0.0027$}
\\

Consolidation
& \neu{$-0.0011 \pm 0.0020$}
& \pos{$0.0062 \pm 0.0014$}
& \pos{$0.0056 \pm 0.0012$}
& \neu{$0.0026 \pm 0.0006$}
& \neu{$0.0020 \pm 0.0007$}
\\

Pneumonia
& \neu{$-0.0037 \pm 0.0017$}
& \pos{$0.0081 \pm 0.0013$}
& \pos{$0.0072 \pm 0.0010$}
& \neu{$0.0038 \pm 0.0005$}
& \neu{$0.0035 \pm 0.0007$}
\\

Atelectasis
& \neu{$-0.0018 \pm 0.0030$}
& \pos{$0.0095 \pm 0.0027$}
& \pos{$0.0101 \pm 0.0025$}
& \neu{$0.0029 \pm 0.0012$}
& \pos{$0.0058 \pm 0.0017$}
\\

Pneumothorax
& \neu{$-0.0048 \pm 0.0020$}
& \pos{$0.0080 \pm 0.0013$}
& \pos{$0.0090 \pm 0.0013$}
& \neu{$0.0018 \pm 0.0005$}
& \neu{$0.0005 \pm 0.0005$}
\\

Pleural Effusion
& \pos{$0.0137 \pm 0.0083$}
& \pos{$0.0465 \pm 0.0060$}
& \pos{$0.0441 \pm 0.0055$}
& \pos{$0.0194 \pm 0.0030$}
& \pos{$0.0228 \pm 0.0030$}
\\

Pleural Other
& \neu{$-0.0034 \pm 0.0006$}
& \neu{$-0.0009 \pm 0.0003$}
& \neu{$-0.0009 \pm 0.0003$}
& \neu{$-0.0004 \pm 0.0002$}
& \neu{$0.0006 \pm 0.0003$}
\\

Fracture
& \neu{$-0.0027 \pm 0.0005$}
& \neu{$0.0003 \pm 0.0002$}
& \neu{$0.0003 \pm 0.0002$}
& \neu{$0.0003 \pm 0.0001$}
& \neu{$0.0016 \pm 0.0003$}
\\

Support Devices
& \nege{$-0.0158 \pm 0.0079$}
& \pos{$0.0208 \pm 0.0056$}
& \pos{$0.0277 \pm 0.0050$}
& \pos{$0.0126 \pm 0.0026$}
& \pos{$0.0183 \pm 0.0030$}
\\

\midrule

\textbf{Average}
& \neu{$-0.0001 \pm 0.0012$}
& \pos{$0.0128 \pm 0.0011$}
& \pos{$0.0135 \pm 0.0008$}
& \pos{$0.0058 \pm 0.0005$}
& \pos{$0.0079 \pm 0.0006$}
\\

\bottomrule
\end{tabular}
}
\end{table}

\subsection{Parametric models across anchor budgets on additional datasets}
\label{app:additional_anchor_budgets}

The main text studies how empirical and parametric anchor models behave as the number of available anchors varies on ImageNet. We observe the same qualitative tradeoff on ImageNet-R and CUB-200. As shown in \Cref{fig:additional_anchor_budgets}, parametric anchor divergences are strongest over much of the limited- and moderate-anchor regime, while softmax improves steadily with anchor coverage and becomes strongest when the full class-anchor set is available. Thus, the anchor-coverage tradeoff observed on ImageNet is not specific to that dataset.

\begin{figure}[t]
    \centering
    \includegraphics[width=\linewidth]
    {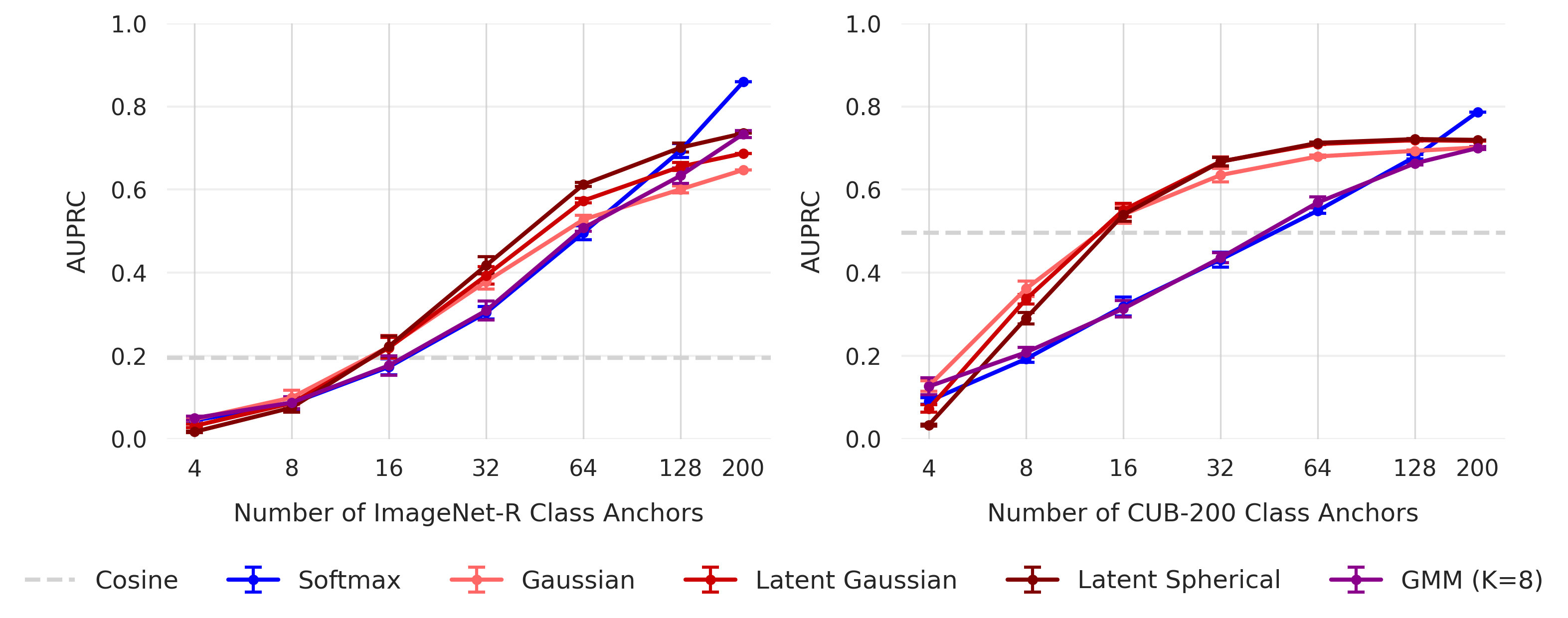}

    \caption{
    \textbf{Parametric structure is useful when anchor coverage is limited.}
    AUPRC as the number of class-name anchors varies on
    \textbf{ImageNet-R} (left) and \textbf{CUB-200} (right).
    Parametric anchor divergences are strongest across much of the limited- and
    moderate-anchor regime, while softmax improves with anchor coverage and is
    strongest when the full class-anchor set is available.
    }

    \label{fig:additional_anchor_budgets}
\end{figure}

\subsection{Exploratory model-complexity diagnostic}
\label{app:exp_model_appropriateness}

We include an exploratory diagnostic on how mixture complexity interacts with anchor-distribution structure. In the anchor-distribution view, parametric models are structured approximations to $q$: increasing flexibility can reduce approximation bias when the anchor distribution is heterogeneous, but can also introduce unnecessary complexity or estimation noise when simpler structure is sufficient.

We compare ImageNet and CUB-200 as illustrative regimes. ImageNet class-name anchors span many coarse semantic groups, while CUB-200 anchors are narrower and concentrated around fine-grained bird categories. \Cref{fig:gmm_model_choice} shows different trends across these datasets: increasing the number of mixture components tends to help on ImageNet, while gains are smaller and less consistent on CUB-200. This diagnostic is not intended as a complete model-selection study, but illustrates that parametric anchor models introduce structure whose usefulness depends on the structure of the anchor set.

\begin{figure}[t]
    \centering
    \includegraphics[width=\linewidth]{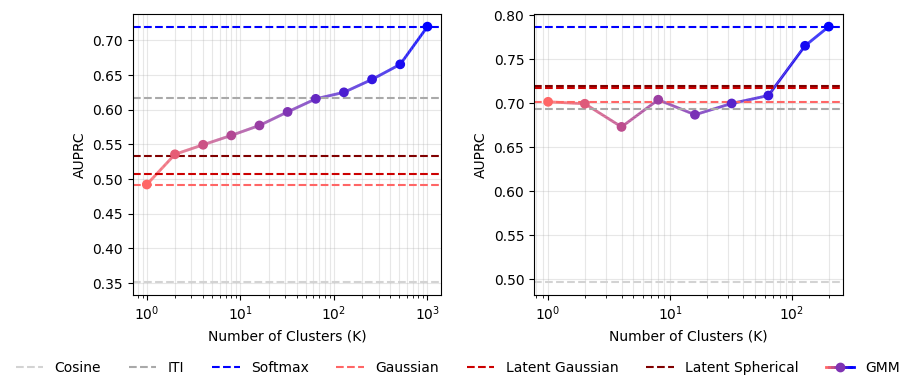}
    \caption{\textbf{Mixture complexity behaves differently across anchor sets.} On ImageNet, increasing the number of GMM components generally improves retrieval, consistent with a heterogeneous class-name anchor distribution. On CUB-200, gains are smaller and less consistent, suggesting that simpler structured geometries may already capture much of the relevant anchor variation.}
    \label{fig:gmm_model_choice}
\end{figure}

\subsection{Computational scaling with anchor-set size}
\label{app:runtime}

Parametric anchor models also provide a computational advantage when the number of observed anchors is large. Let $Q$ denote the number of queries, $G$ the number of gallery items, $N$ the number of anchors, and $d$ the
embedding dimension. 
A direct implementation of softmax anchor divergence constructs an $N$-dimensional distribution for each image, giving computational cost $O(QGN + (Q + G)Nd)$.

For the symmetrized divergence used in our experiments, the subsequent query--gallery comparison can be computed exactly from cached $d$-dimensional statistics. Writing
$\eta=\lambda/\tau$ and
$\mu(\lambda)=\nabla\kappa_{\hat q_\Gamma}(\eta)$, we have
\begin{equation}
    d_{\mathrm{sym}}(\lambda_q,\lambda_g) = \frac{1}{2}(\eta_q-\eta_g)^\top\bigl(\mu(\lambda_q)-\mu(\lambda_g)\bigr).
\end{equation}
Thus, after caching $\mu(\lambda)$, softmax anchor divergence has total cost
$$
    O((Q+G)Nd + QGd),
$$
so dependence on $N$ remains in the gallery and query dependent preprocessing step. 
In contrast, a fitted parametric anchor model summarizes the observed anchors by a fixed set of parameters. Once this fitting step is complete, query and gallery preprocessing and pairwise evaluation need not depend on the original anchor-set size.

\begin{table}[t]
\centering
\caption{
\textbf{Runtime scaling with anchor-set size.}
Wall-clock retrieval time (in seconds) on a single NVIDIA H100 NVL GPU for 10,000 queries and 50,000 gallery images.
Runtimes include fitting, preprocessing, and retrieval and are averaged over 10 runs following 3 warm-up runs.\\
}
\label{tab:runtime_scaling}

\begin{tabular}{lccccc}
\toprule
\textbf{Method}
& $\mathbf{2^{10}}$
& $\mathbf{2^{12}}$
& $\mathbf{2^{14}}$
& $\mathbf{2^{16}}$
& $\mathbf{2^{17}}$ \\
\midrule
Cosine
& 0.152
& 0.152
& 0.152
& 0.151
& 0.152 \\
Softmax
& 0.216
& 0.281
& 0.538
& 1.622
& 3.066 \\
Gaussian
& 0.364
& 0.363
& 0.367
& 0.370
& 0.375 \\
\bottomrule
\end{tabular}
\end{table}

We examine this empirically by varying the number of anchors and timing retrieval for 10,000 queries against 50,000 gallery image embeddings on a single NVIDIA H100 NVL GPU. Runtimes include model fitting, and retrieval, and are averaged over 10 runs following 3 warm-up runs.
For softmax anchor divergence, anchors are streamed in chunks when computing the cached statistics, so additional peak GPU memory does not grow substantially with $N$.
\Cref{tab:runtime_scaling} shows that softmax runtime increases with anchor-set size as anchor-dependent preprocessing becomes appreciable, whereas Gaussian anchor divergence remains nearly constant after the anchor distribution has been fit.

\subsection{Qualitative object--style retrieval examples}
\label{app:qualitative_semantic_control}

\Cref{fig:qualitative_semantic_control} shows additional object- and style-anchor retrieval examples. Object anchors retrieve images with the same object across styles, while style anchors retrieve images with the same style across object categories. This pattern appears for both empirical softmax and parametric anchor divergences, suggesting that parametric models preserve the semantic coordinate selected by the anchors.

\begin{figure}[h]
    \centering
    \includegraphics[width=0.9\textwidth]{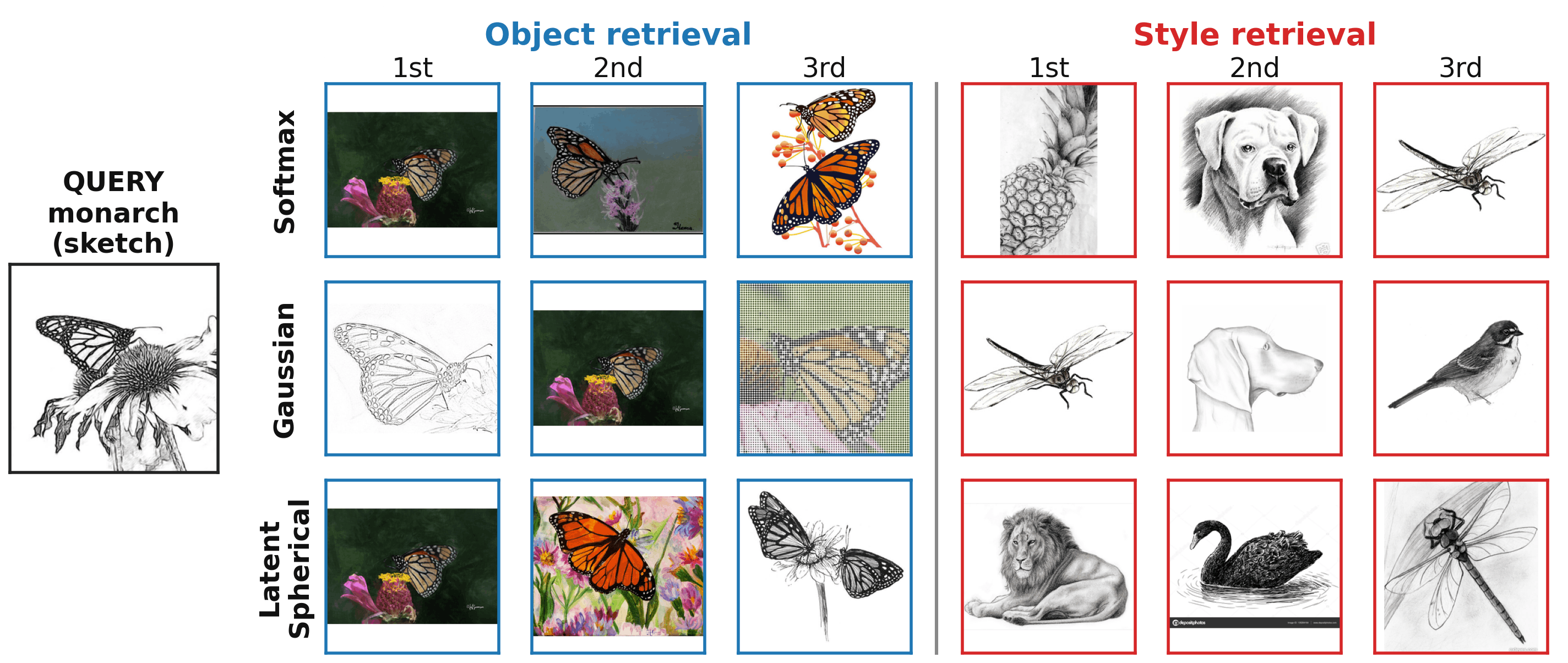}
    \includegraphics[width=0.9\textwidth]{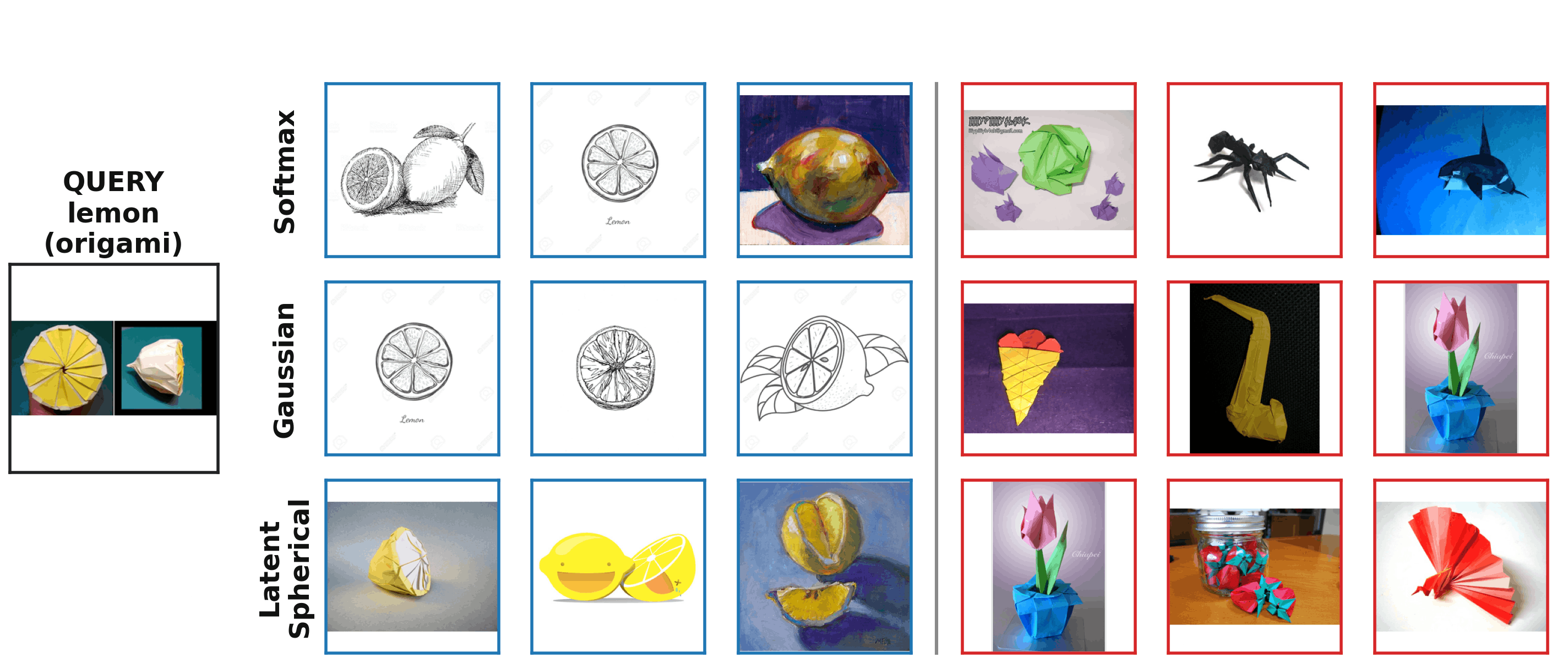}
    \includegraphics[width=0.9\textwidth]{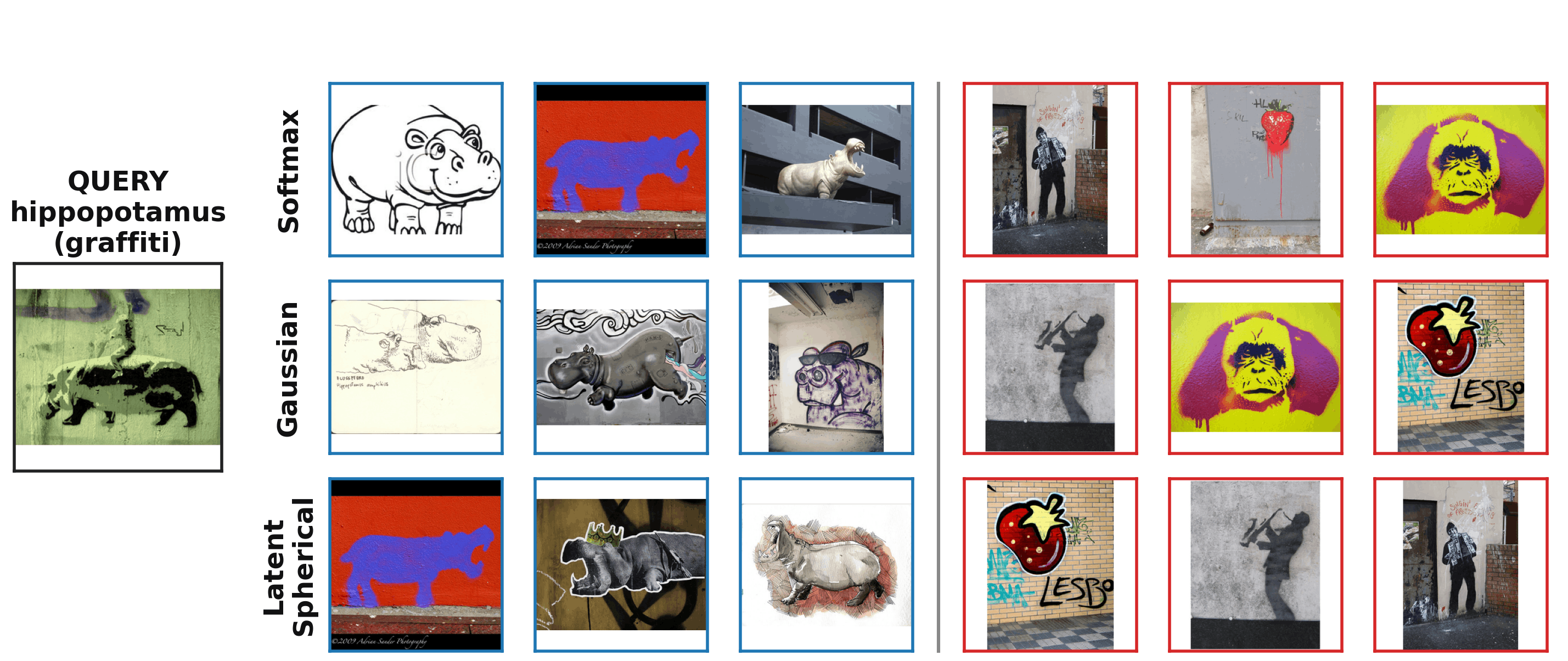}
    \caption{\textbf{Qualitative examples of anchor-controlled retrieval.} Blue boxes show object-anchor retrieval, which prioritizes object identity across styles; red boxes show style-anchor retrieval, which prioritizes style across object categories. Rows compare some of the different anchor divergences.}
    \label{fig:qualitative_semantic_control}
\end{figure}

\end{document}

%% file: preamble/preamble.tex
\usepackage[utf8]{inputenc} 
\usepackage[T1]{fontenc}    

\usepackage[bitstream-charter]{mathdesign}
\usepackage{amsmath}
\usepackage[scaled=0.92]{PTSans}

\usepackage[
  paper  = letterpaper,
  left   = 1.65in,
  right  = 1.65in,
  top    = 1.0in,
  bottom = 1.0in,
  ]{geometry}

\usepackage[usenames,dvipsnames,table]{xcolor}
\definecolor{shadecolor}{gray}{0.9}

\usepackage[final,expansion=alltext]{microtype}
\usepackage[english]{babel}
\usepackage[parfill]{parskip}
\usepackage{afterpage}
\usepackage{framed}

{\endMakeFramed}

\usepackage{lineno}

\usepackage{ragged2e}

\newcounter{parcount}

\usepackage{graphicx}
\usepackage{wrapfig}
\usepackage[labelfont=bf,font=footnotesize,width=.9\textwidth]{caption}
\usepackage[format=hang]{subcaption}

\usepackage{booktabs,multirow,multicol}       

\usepackage[algoruled,algo2e]{algorithm2e}
\usepackage{listings}
\usepackage{fancyvrb}
\fvset{fontsize=\normalsize}

\usepackage[colorlinks,linktoc=all]{hyperref}
\usepackage[all]{hypcap}
\hypersetup{citecolor=MidnightBlue}
\hypersetup{linkcolor=MidnightBlue}
\hypersetup{urlcolor=MidnightBlue}

\usepackage[nameinlink]{cleveref}

\usepackage[acronym,nowarn]{glossaries}

\lstdefinestyle{mystyle}{
    commentstyle=\color{OliveGreen},
    keywordstyle=\color{BurntOrange},
    numberstyle=\tiny\color{black!60},
    stringstyle=\color{MidnightBlue},
    basicstyle=\ttfamily,
    breakatwhitespace=false,
    breaklines=true,
    captionpos=b,
    keepspaces=true,
    numbers=left,
    numbersep=5pt,
    showspaces=false,
    showstringspaces=false,
    showtabs=false,
    tabsize=2
}

%% file: preamble/preamble_math.tex
\usepackage{amsthm}

\usepackage{centernot}
\usepackage{nicefrac}       
\usepackage{mathtools}
\usepackage{amsbsy}
\usepackage{amstext}
\usepackage{thmtools}
\usepackage{thm-restate}

\begingroup
    \makeatletter
    \@for\theoremstyle:=definition,remark,plain\do{%
        \expandafter\g@addto@macro\csname th@\theoremstyle\endcsname{%
            \addtolength\thm@preskip\parskip
            }%
        }
\endgroup

\crefname{lemma}{lemma}{lemmas}
\Crefname{lemma}{Lemma}{Lemmas}
\crefname{thm}{theorem}{theorems}
\Crefname{thm}{Theorem}{Theorems}
\crefname{prop}{proposition}{propositions}
\Crefname{prop}{Proposition}{Propositions}
\crefname{assumption}{assumption}{assumptions}
\crefname{assumption}{Assumption}{Assumptions}

\renewcommand{\mid}{~\vert~}

\usepackage{booktabs,arydshln}
\makeatletter
\def\adl@drawiv#1#2#3{%
        \hskip.5\tabcolsep
        \xleaders#3{#2.5\@tempdimb #1{1}#2.5\@tempdimb}%
                #2\z@ plus1fil minus1fil\relax
        \hskip.5\tabcolsep}
\newcommand{\cdashlinelr}[1]{%
  \noalign{\vskip\aboverulesep
           \global\let\@dashdrawstore\adl@draw
           \global\let\adl@draw\adl@drawiv}
  \cdashline{#1}
  \noalign{\global\let\adl@draw\@dashdrawstore
           \vskip\belowrulesep}}
\makeatother

\renewcommand{\epsilon}{\varepsilon}

\declaretheorem[style=plain,name=Theorem]{theorem}

\declaretheorem[style=definition,sibling=theorem,name=Definition]{definition}

\declaretheorem[style=definition,sibling=theorem,name=Example]{example}

\newenvironment{example*}
 {\pushQED{\qed}\example}
 {\popQED\endexample}
\numberwithin{equation}{section}

%% file: preamble/definitions_basic.tex
\DeclarePairedDelimiterX\Set[1]{\lbrace}{\rbrace}%
{  #1 }



%% file: preamble/minimalist_biblatex.tex
\usepackage{csquotes}
\usepackage[%
minnames=1,maxnames=99,maxcitenames=2,
style=alphabetic,
doi=false,
url=false,
giveninits=true,
hyperref,
natbib,
backend=bibtex,
sorting=nyt,
backref=true
]{biblatex}%
\renewbibmacro{in:}{%
  \ifentrytype{article}{}{\printtext{\bibstring{in}\intitlepunct}}}

\renewbibmacro*{journal}{%
  \iffieldundef{journaltitle}
    {}
    {\printtext[journaltitle]{%
       \printfield[noformat]{journaltitle}%
       \setunit{\subtitlepunct}%
       \printfield[noformat]{journalsubtitle}}}}

\DeclareFieldFormat{sentencecase}{\MakeSentenceCase*{#1}}

\renewbibmacro*{title}{%
  \ifthenelse{\iffieldundef{title}\AND\iffieldundef{subtitle}}
    {}
    {\ifthenelse{\ifentrytype{article}\OR\ifentrytype{inbook}%
      \OR\ifentrytype{incollection}\OR\ifentrytype{inproceedings}%
      \OR\ifentrytype{inreference}\OR\ifentrytype{misc}}
      {\printtext[title]{%
        \printfield[sentencecase]{title}%
        \setunit{\subtitlepunct}%
        \printfield[sentencecase]{subtitle}}}%
      {\printtext[title]{%
        \printfield[titlecase]{title}%
        \setunit{\subtitlepunct}%
        \printfield[titlecase]{subtitle}}}%
     \newunit}%
  \printfield{titleaddon}}

\AtEveryBibitem{%
\ifentrytype{article}{
    \clearfield{urldate}%
    \clearfield{eprint}
    \clearfield{eid}
}{}
\ifentrytype{book}{
    \clearfield{url}%
    \clearfield{urldate}%
    \clearfield{eprint}
}{}
\ifentrytype{collection}{
    \clearfield{url}%
    \clearfield{urldate}%
    \clearfield{eprint}
}{}
\ifentrytype{incollection}{
    \clearfield{url}%
    \clearfield{urldate}%
    \clearfield{eprint}
}{}
}

\AtEveryBibitem{
    \clearfield{pages}
    \clearfield{review}%
    \clearfield{series}
    \clearfield{volume}
    \clearfield{month}
    \clearfield{isbn}
    \clearfield{issn}
    \clearlist{location}
    \clearfield{series}
    \clearlist{publisher}
    \clearname{editor}
}{}